\documentclass{article}
\usepackage{latexml}
\iflatexml
  \usepackage{times,natbib}
\else
  \usepackage{iclr2027_conference,times}
\fi

\usepackage{amsmath,amsfonts,bm}

\def\eqref#1{equation~\ref{#1}}

\def\1{\bm{1}}
\newcommand{\train}{\mathcal{D}}

\newcommand{\test}{\mathcal{D_{\mathrm{test}}}}

\DeclareMathAlphabet{\mathsfit}{\encodingdefault}{\sfdefault}{m}{sl}
\SetMathAlphabet{\mathsfit}{bold}{\encodingdefault}{\sfdefault}{bx}{n}

\usepackage{hyperref}
\usepackage{url}
\usepackage{booktabs}
\usepackage{amsfonts}
\usepackage{nicefrac}
\usepackage{microtype}
\usepackage{xcolor}
\usepackage{graphicx}
\usepackage{subcaption}
\usepackage{multirow}
\usepackage{xspace}
\usepackage{enumitem}
\usepackage{amsmath}
\usepackage{tabularx}
\usepackage{makecell}
\usepackage{amsthm}
\usepackage{amssymb}
\usepackage{commands}
\usepackage{thmtools}
\usepackage[labelfont=bf,textfont=normal]{caption}
\usepackage{array}
\usepackage{wrapfig}
\usepackage{float}

\title{Right Direction, Wrong Step: Geometric Analysis of Finite-Step Failure in Looped Transformers}

\iflatexml\else\iclrfinalcopy\fi
\usepackage{authblk}
\author[1]{Zhihao Guo}
\author[2]{Zonghan Wu}
\author[3]{Haizhou Du}
\author[1]{Huan Huo\thanks{Corresponding author: \href{mailto:huan.huo@uts.edu.au}{\nolinkurl{huan.huo@uts.edu.au}}.}}
\author[2]{Yilei Shao\thanks{Corresponding author: \href{mailto:yileishao@sem.ecnu.edu.cn}{\nolinkurl{yileishao@sem.ecnu.edu.cn}}.}}
\author[4]{Athanasios V. Vasilakos}
\author[2,5]{Qingsong Wen}
\affil[1]{University of Technology Sydney, Australia}
\affil[2]{East China Normal University, China}
\affil[3]{Shanghai University of Electric Power, China}
\affil[4]{University of Agder, Norway}
\affil[5]{Squirrel AI Learning, USA}
\iflatexml
  \affil[2,5]{East China Normal University, China; Squirrel AI Learning, USA}
\fi
\date{}
\makeatletter
\iflatexml
\else
  \iclrfinalcopy
  \gdef\@author{%
    {\fontsize{10.5}{14}\selectfont
      \textbf{Zhihao Guo}\textsuperscript{1}\qquad
      \textbf{Zonghan Wu}\textsuperscript{2}\qquad
      \textbf{Haizhou Du}\textsuperscript{3}\qquad
      \textbf{Huan Huo}\textsuperscript{1, *}\qquad
      \textbf{Yilei Shao}\textsuperscript{2, *}\par
      \textbf{Athanasios V. Vasilakos}\textsuperscript{4}\qquad
      \textbf{Qingsong Wen}\textsuperscript{2, 5}\par}
    \vspace{0.7em}
    {\fontsize{9}{11.5}\selectfont
      \textsuperscript{1}University of Technology Sydney, Australia\quad
      \textsuperscript{2}East China Normal University, China\par
      \textsuperscript{3}Shanghai University of Electric Power, China\par
      \textsuperscript{4}University of Agder, Norway\quad
      \textsuperscript{5}Squirrel AI Learning, USA\par}
  }
  \hypersetup{
    pdftitle={\@title},
    pdfauthor={Zhihao Guo, Zonghan Wu, Haizhou Du, Huan Huo, Yilei Shao, Athanasios V. Vasilakos, Qingsong Wen}
  }
  \renewcommand{\maketitle}{%
    \begingroup
      \centering
      \setlength{\parskip}{0pt}%
      \setlength{\parindent}{0pt}%
      {\normalfont\fontsize{17}{21}\selectfont\bfseries
        \hyphenpenalty=10000\exhyphenpenalty=10000
        \@title\par}
      \vspace{1.1em}
      {\normalfont\@author\par}
    \endgroup
    \begingroup
      \renewcommand{\thefootnote}{\fnsymbol{footnote}}%
      \footnotetext[1]{Correspondence: huan.huo@uts.edu.au
        \& yileishao@sem.ecnu.edu.cn}%
    \endgroup
    \vspace{1.0em}
    \fancyhead{}%
    \renewcommand{\headrulewidth}{0pt}%
  }

\fi
\makeatother
\begin{document}

\maketitle

\begin{abstract}
Looped Transformers offer a parameter-efficient route to test-time
scaling by reusing shared layers for iterative latent reasoning.
However, additional iterations can reduce support for a reference answer, leaving unclear
whether an update's direction is locally unhelpful or its full
displacement moves too far.
We study this distinction by analysing reference utility, which measures
this support, along the model's own update direction, varying the fraction of the proposed
displacement supplied to the readout.
This reveals finite-step failures in which a locally improving direction
produces a harmful full update.
A pathwise curvature decomposition characterises how initial progress
is lost, while a local quadratic model predicts full-step gains
and useful step scales.
Bounds based on accumulated curvature variation characterise the
approximation error of these predictions.
Experiments across two model families reveal this separation on
mathematical and commonsense tasks. A fixed quarter step recovers
positive gains in reference utility for 72.2--83.2\% of selected failures across
four settings.
These findings identify a mismatch between update direction and step
scale as a mechanism of lost progress, explaining how some harmful updates
retain useful computation.
\end{abstract}

\section{Introduction}
\label{sec:intro}

Looped Transformers models repeatedly apply shared Transformer layers,
making computational depth an inference-time resource~\citep{DBLP:conf/icml/GiannouRS0LP23}.  Following
Universal Transformers~\citep{DBLP:conf/iclr/DehghaniGVUK19}, Huginn and
Ouro bring recurrent computation to language-model
pretraining~\citep{NeurIPS2025_70913,zhu_scaling_2026}, while
pretrained models acquire recurrence through parameter-sharing
conversions~\citep{DBLP:conf/iclr/BaeFHJ0S25} or retrofitting with
recurrence curricula~\citep{mcleish2025retrofitting}.
Expressivity analyses~\citep{DBLP:conf/icml/GiannouRS0LP23,ICML2025_46034}
and studies of in-context algorithm
learning~\citep{DBLP:conf/iclr/Yang0NP24,DBLP:conf/icml/GatmirySRJK24}
explain how repeated blocks support iterative computation.  Empirical
work connects recurrent depth to reasoning and generalisation across
problem difficulty or
length~\citep{DBLP:conf/nips/BansalSBEHGG22,DBLP:conf/iclr/FanDRL25,DBLP:conf/iclr/SaunshiDLKR25}.
Connecting this extra computation to task progress requires understanding
its learned updates.

\newcommand{\introConceptFigure}{%
\begin{wrapfigure}[13]{r}{0.43\textwidth}
  \centering
  \vspace{-5mm}
  \includegraphics[width=\linewidth]{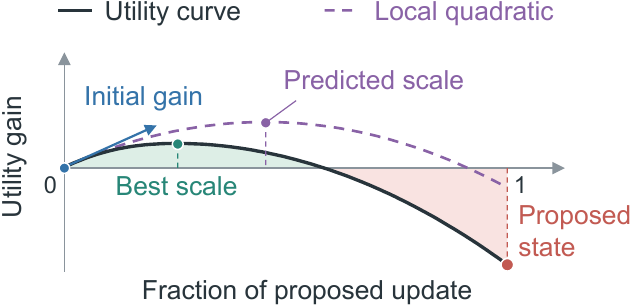}
  \captionsetup{font=footnotesize,skip=1pt}
  \caption{\textbf{A shorter step can recover progress.} Utility gain measures the change in support for the reference answer relative to the current state (0). The proposed state is at (1).}
  \label{fig:intro-concept}
\end{wrapfigure}%
}
\ifdefined\arxivVersion\else\introConceptFigure\fi

Recent work~\citep{yang_stabilizing_2026} brings the fixed-point stability
perspective developed for deep equilibrium
models~\citep{DBLP:conf/nips/BaiKK19,DBLP:conf/icml/BaiKK21} to looped
language models, combining Jacobian regularisation and random loop
sampling to improve test-time scaling.  Input-dependent depth
selection~\citep{DBLP:conf/iclr/ElbayadGGA20} and time- and
step-conditioned updates~\citep{ICLR2026_96305} adapt computation to
different inputs and budgets, while learned correctness signals guide
latent optimisation~\citep{du2026latent}. The effect of a recurrent
update on support for a reference answer depends on both its direction
and magnitude, leaving unclear whether scaling down the same update
could strengthen that support.

\ifdefined\arxivVersion\introConceptFigure\fi
\vspace{4mm}
Along the segment between the current and proposed states, support for
the reference answer can have a positive initial slope yet end below
its starting level, producing the \emph{finite-step failure}
illustrated in Figure~\ref{fig:intro-concept}.  A sufficiently negative
accumulated curvature correction can outweigh the first-order gain.
A shorter displacement along the same direction can then reach an
intermediate state that supports the reference answer more strongly
than both the current and proposed states.
The initial slope and curvature suggest how far to move,
while variation in curvature over the segment determines how faithfully
this local prediction describes the change in support along the update.

We investigate this separation through pathwise analysis and
fixed-direction interventions, using a reference utility that measures
support for the reference answer. Finite-step
failures occur in all nine model--task conditions, including those with
positive mean gain.  Local quadratic geometry predicts full-step gain
and useful update scales, connecting the initial direction to
the progress available within the displacement.  A fixed quarter
step restores positive utility in 72.2--83.2\% of selected failures
across four settings, and quadratic selection improves recovery and
mean gain on an additional full test set.  To characterise the error
of the local predictions, we use bounds based on integrated
curvature variation along the same path.  Interventions measure reference
utility after proposal computation and before further recurrence.
\emph{Our contributions are}:
\begin{enumerate}[nosep,leftmargin=1.5em]
  \item \textbf{Finite-step failure mechanism.}
  Across nine model--task conditions, we identify harmful recurrent updates
  whose directions remain locally improving, and characterise their loss
  of progress through an exact pathwise curvature decomposition.

  \item \textbf{Local geometry of useful update scales.}
  We use a local quadratic model of reference utility to predict useful
  scales along observed recurrent updates, and use integrated curvature
  variation to bound scale error and regret relative to the continuous optimum.

  \item \textbf{Progress recovery by step contraction.}
  We demonstrate recovery of finite-step failures through fixed
  and curvature-selected step contractions, and use endpoint-oracle comparisons
  to quantify the additional value of intermediate scales.
\end{enumerate}

\section{Related Work}
\label{sec:relate}

\textbf{Recurrent Computation and Overthinking.}
Shared recurrent computation supports depth and length
generalisation~\citep{DBLP:conf/nips/SchwarzschildBGHVGG21,DBLP:conf/nips/BansalSBEHGG22,DBLP:conf/iclr/FanDRL25,DBLP:conf/iclr/SaunshiDLKR25},
while adaptive computation methods learn how much processing to allocate
to each input~\citep{graves2016act,banino2021pondernet,DBLP:conf/iclr/ElbayadGGA20}.
The benefits of additional computation, however, need not persist
with increasing depth. Shallow-Deep Networks identify destructive
overthinking when deeper processing turns correct predictions into
errors~\citep{DBLP:conf/icml/KayaHD19}, and Think-at-Hard observes
analogous reversals during latent
iterations~\citep{fu2026thinkatharddynamicloopedtransformers}.
To understand how additional recurrence loses task progress, we
examine whether a harmful update follows an unhelpful direction or
moves too far along a locally improving one.

\textbf{Stability and Step Control.}
Building on deep equilibrium models~\citep{DBLP:conf/nips/BaiKK19,DBLP:conf/icml/BaiKK21},
STARS~\citep{yang_stabilizing_2026} stabilises recurrence through
Jacobian regularisation and random loop sampling.
This dynamical perspective also informs when to stop recurrence:
remaining path length and decoder margin yield conditions for answer
stability~\citep{viakhirev2026thinkshallowsolvedeep}, while differences
between successive updates provide an early-exit
signal~\citep{pappone2025twoscalelatentdynamicsrecurrentdepth}.
Alongside stopping decisions, recurrent models regulate how individual
updates are applied through residual-dependent
damping~\citep{movahedi2026fixedpointreasonersstableadaptive},
input-dependent channel-wise gating~\citep{park2026loopusrecastingpretrainedllms},
and time- and step-conditioning~\citep{ICLR2026_96305}.
These approaches highlight the role of update control in recurrent computation. To understand why a smaller step can recover task progress, we analyse how reference utility changes along a proposed update while holding its direction and readout fixed.

\textbf{Task Feedback and Update Geometry.}
Task feedback~\citep{du2026latent,ICLR2026_31260,DBLP:conf/aaai/WangWYBGDLF26}
and representation geometry~\citep{NeurIPS2025_52493} have been used
to guide internal-state refinement.
Recent analysis of Tiny Recursive Models derives reference-token
loss derivatives along an interpolation between pre- and post-update
distributions~\citep{asadulaev2026latentreasoningtrmssecretly}.
Although our state interpolation induces the same distribution path
under affine logit readouts, we extend the analysis to how accumulated
curvature can reverse local utility improvement over a full update,
and test whether shortening that update increases reference utility
above its pre-update value.

\section{Method}
\label{sec:method}

Each application of a looped Transformer's shared block determines
an update direction and a displacement magnitude. For a frozen model
and a fixed readout, we define reference utility along the proposed
update in \S~\ref{sec:task-utility}. We then use pathwise curvature
to explain how local improvement can become full-step harm
(\S~\ref{sec:finite-progress}), and analyse useful step scales
through a local quadratic approximation and its error bounds
(\S~\ref{sec:optimal-step}). Finally, fixed-direction interventions
test whether shorter displacements recover immediate reference utility
(\S~\ref{sec:contraction}).

\subsection{Recurrent Updates and Task Utility}
\label{sec:task-utility}

Looped language models repeatedly apply a shared
transformation~\citep{NeurIPS2025_70913,zhu_scaling_2026}.
For a frozen model and a fixed input, write the recurrence as $H_{t+1}=F(H_t)$,
where $H_t\in\mathcal H$ is the hidden state after $t=0,1,\ldots$
iterations, $H_0$ is the initial state, and $\mathcal H$ is a
finite-dimensional real inner product space.  The map
$F:\mathcal H\to\mathcal H$ is fixed.  We analyse a consecutive pair
$H=H_t$, $H^+=H_{t+1}$ and its displacement
\begin{equation}
  \label{eq:recurrent-update}
  H^+:=F(H),
  \qquad
  D:=H^+-H.
\end{equation}
The fixed output layers, or readout, map hidden states to output
scores.  A scalar task utility $U:\Omega\to\mathbb R$ evaluates
these scores against a fixed reference, where
$\Omega\subseteq\mathcal H$ is an open set containing the segment from
$H$ to $H^+$.  An increase in $U$ defines task progress.  The state
representation, readout, and reference remain fixed throughout each
comparison, and Appendix~\ref{app:utilities} defines the utilities.

To isolate the effect of displacement magnitude, we preserve the
direction proposed by the recurrent model. For a nonzero update,
every displacement in the same direction is a positive scalar
multiple of that update. We define the interpolated state
and its utility gain at a dimensionless step scale $\alpha$:
\begin{equation}
  \label{eq:utility-path}
  \gamma(\alpha):=H+\alpha D,
  \qquad
  \phi(\alpha):=U(\gamma(\alpha))-U(H),
  \qquad \alpha\in[0,1].
\end{equation}
Thus $\phi(0)=0$ and $\phi(1)=\Delta U:=U(H^+)-U(H)$.
An additional recurrent iteration takes the full displacement.  For $D\neq0$,
intermediate scales retain its direction while varying the state supplied to the
readout, after both endpoints have been computed.
We relate this path to its local
derivatives under the following assumption.

\begin{assumption}[Readout Regularity]
  \label{ass:readout-regularity}
  The utility $U$ is twice continuously differentiable on the open set
  $\Omega$ containing $\{H+\alpha D:\alpha\in[0,1]\}$.
\end{assumption}

The evaluated readouts consist of smooth operations under fixed
attention masks, so their real-valued utilities satisfy this
regularity (Appendix~\ref{app:state-boundaries}).
Since the proposed displacement is held fixed, no smoothness
assumption on the recurrent map $F$ is required.
Assumption~\ref{ass:readout-regularity} gives $\phi\in C^2(\mathcal O)$
on an open neighbourhood $\mathcal O$ of $[0,1]$.
The initial slope $A$ and directional curvature $2Q$ are
\begin{equation}
  \label{eq:directional-coefficients}
  A:=\phi'(0)=\bigl\langle\nabla U(H),D\bigr\rangle_{\mathcal{H}},
  \qquad
  Q:=\frac{1}{2}\phi''(0)
    =\frac{1}{2}\bigl\langle D,\nabla^2 U(H)[D]\bigr\rangle_{\mathcal{H}}.
\end{equation}
These coefficients measure the utility response to the shared block's
own displacement (derivation in Appendix~\ref{app:directional-derivatives}).
The direction is supplied by learned recurrence, and the utility
derivatives assess its alignment with the task.

The sign of $A$ describes the initial effect of following the update,
whereas $\Delta U$ measures the result of the full displacement.
We use their joint signs to distinguish two kinds of harmful update.

\begin{definition}[Finite-Step Failure]
  \label{def:finite-failure}
  An update is progressing if $\phi(1)>0$ and neutral if
  $\phi(1)=0$.  A harmful update ($\phi(1)<0$) is a
  directional failure if $A\leq0$ and a finite-step failure
  if $A>0$.
\end{definition}

In a finite-step failure, the shared block proposes a direction that
initially improves utility, but taking its full displacement lowers
utility at the next recurrent state.  The curvature accumulated within
this update determines how the initial improvement is lost.

\subsection{Direction--Step Compatibility}
\label{sec:finite-progress}

A scale is compatible with task progress when $\phi(\alpha)\geq0$.
The integral form of Taylor's theorem expresses this gain as the initial
linear gain plus a pathwise curvature correction.

\begin{proposition}[Pathwise Progress Decomposition]
  \label{prop:pathwise-progress}
  Let $\mathcal{O}\subseteq\mathbb{R}$ be open with
  $[0,1]\subseteq\mathcal{O}$.  If $\phi\in C^2(\mathcal{O})$ and
  $\phi(0)=0$, then for every $\alpha\in[0,1]$,
  \begin{equation}
    \label{eq:pathwise-progress}
    \phi(\alpha)
      =\alpha A+\int_0^\alpha(\alpha-s)\phi''(s)\,ds.
  \end{equation}
  In particular, a finite-step failure satisfies
  $\int_0^1(1-s)\phi''(s)\,ds<-A<0$.
\end{proposition}

The proof is given in Appendix~\ref{app:progress-proof}.
At $\alpha=1$, the integral corrects the local gain $A$ for directional
curvature accumulated across the full update.  A finite-step failure
occurs when this correction outweighs the positive linear term.

For an affine logit path, the curvature correction admits an exact
interpretation in terms of the output distributions.  Fix a question $i$
with a finite nonempty set $M_i$ of reference-token indices.
Let $p_j(\alpha)$ denote the softmax distribution predicting reference
token $j\in M_i$ over a finite nonempty vocabulary.

\begin{proposition}[Exact Endpoint Decomposition]
  \label{prop:linear-endpoint}
  Suppose the logits for every $j\in M_i$ are affine in $\alpha$
  and $U$ is the mean reference-token log probability.  Then
  \begin{equation}
    \label{eq:linear-endpoint}
    \Delta U=A-\mathcal{C},
    \qquad
    \mathcal{C}:=\frac{1}{|M_i|}\sum_{j\in M_i}
      D_{\mathrm{KL}}(p_j(0)\|p_j(1))\geq0.
  \end{equation}
  A finite-step failure occurs exactly when $0<A<\mathcal{C}$.
  Moreover, $\mathcal{C}=-\int_0^1(1-s)\phi''(s)\,ds$.
\end{proposition}

Here $D_{\mathrm{KL}}$ denotes Kullback--Leibler divergence.
The proof is in Appendix~\ref{app:endpoint-proof}.
The full update loses utility when the distributional correction exceeds
the first-order gain.  Since this correction uses only the
endpoint distributions, it also provides a separate numerical check of
the pathwise decomposition.
For a general readout, we obtain an explicit sufficient range of improving
scales by bounding negative directional curvature uniformly along the
segment.

\begin{corollary}[Beneficial Step Scales]
  \label{cor:beneficial-step}
  Under Proposition~\ref{prop:pathwise-progress}, suppose $A>0$ and
  $\phi''(\alpha)\geq-L_D$ on $[0,1]$ for a constant $L_D\geq0$.  Then
  for every $\alpha\in[0,1]$,
  \begin{equation}
    \label{eq:beneficial-step}
    \phi(\alpha)\geq\alpha A-\frac{L_D}{2}\alpha^2.
  \end{equation}
  For $L_D>0$, every $0<\alpha\leq1$ with $\alpha<2A/L_D$ has positive
  utility gain.  For $L_D=0$, every $\alpha\in(0,1]$ has positive gain.
\end{corollary}

The proof also appears in Appendix~\ref{app:progress-proof}.
The sufficient range turns the directional diagnosis into an intervention:
it specifies how far the proposed update can move while improving utility.
A valid $L_D$ requires a lower bound on directional curvature along the
whole segment.  The experiments obtain this bound from an FP64 derivative
bound for Ouro's affine readout and have no such bound for Huginn
(Appendix~\ref{app:exp-fp64}).
The path defines the progress boundary, the endpoint of its
largest initial interval of nonnegative gain:
\begin{equation}
  \label{eq:progress-boundary}
  r_1:=\sup\Bigl\{r\in[0,1]\,:\,\phi(\beta)\geq0
                       \quad\text{for every }\beta\in[0,r]\Bigr\}.
\end{equation}
For a finite-step failure, the positive initial derivative and negative
endpoint give $0<r_1<1$.  The argument is in
Appendix~\ref{app:boundary-proof}.
This boundary describes how far progress persists.  Selecting a scale
further requires locating the utility maximum.

\subsection{Predicting the Utility-Maximising Scale}
\label{sec:optimal-step}

\begingroup
\interlinepenalty=10000
\postdisplaypenalty=10000
The quadratic model predicts $\widehat\alpha$ for the path's continuous
optimum $\alpha^\star$:
\begin{equation}
  \label{eq:quadratic-model}
  q(\alpha):=A\alpha+Q\alpha^2,
  \qquad
  \alpha^\star:=\min\operatorname*{arg\,max}_{\alpha\in[0,1]}\phi(\alpha),
  \qquad
  \widehat\alpha:=\min\operatorname*{arg\,max}_{\alpha\in[0,1]}q(\alpha).
\end{equation}
\endgroup
Both optimisation problems use the smallest maximiser to resolve ties.
The same local model gives two predictions:
$q(1)=A+Q$ for the gain of the full update and $\widehat\alpha$ for
the scale at which the gain is largest.
For $Q<0$, the predicted scale is
\begin{equation}
  \label{eq:predicted-step}
  \widehat\alpha=
  \operatorname{clip}_{[0,1]}\!\left(-\frac{A}{2Q}\right).
\end{equation}
Here $\operatorname{clip}_{[0,1]}(x):=\min\{1,\max\{0,x\}\}$.
For $Q\geq0$, $\widehat\alpha$ is $1$ if $A+Q>0$ and $0$ otherwise.

For $A>0,Q<0$, the ratio $-A/(2Q)$ balances the initial slope against
its decrease under the quadratic model.  Clipping gives the best scale
within the original displacement.  The polynomial's positive root,
$r^{(2)}=-A/Q$, predicts where that gain returns to zero.
It is twice the unclipped maximising scale and may lie beyond the search
interval.  The scale bounds below do not apply to this root.  Under a
two-sided curvature bound, Appendix~\ref{app:boundary-proof} places
$r^{(2)}$ and $r_1$ in a common interval, and the experiments compare
$r^{(2)}$ with the sampled path.  Existence of the smallest maximisers and the derivation of
Eq.~(\ref{eq:predicted-step}) are given in
Appendix~\ref{app:quadratic-max}.

The accuracy of this prediction depends on how the directional slope
changes along the segment.  We measure the accumulated departure from
the initial curvature by
\begin{equation}
  \label{eq:curvature-deviation}
  C(a):=\int_0^a |\phi''(s)-\phi''(0)|\,ds,
  \qquad a\in[0,1].
\end{equation}
This quantity measures variation around the curvature already represented
by $q$.  A curved path can still have small approximation error: constant
curvature gives $C(1)=0$ and makes $q$ exact, as shown in
Appendix~\ref{app:integrated-edge-cases}.
Under Assumption~\ref{ass:readout-regularity}, the full-update prediction
error satisfies
\begin{equation}
  \label{eq:endpoint-prediction-bound}
  |\Delta U-(A+Q)|\leq C(1).
\end{equation}
In particular, $A+Q$ has the same sign as $\Delta U$ whenever
$|A+Q|>C(1)$ (proof in Appendix~\ref{app:endpoint-prediction-proof}).
For $Q<0$, this deviation also bounds scale error and regret relative
to the initial curvature.

\begin{theorem}[Integrated Curvature Approximation]
  \label{thm:integrated-step}
  Let $\phi\in C^2(\mathcal O)$ on an open neighbourhood of $[0,1]$,
  and set $A=\phi'(0)$ and $Q=\phi''(0)/2<0$.  Set $\kappa:=-2Q>0$ and use
  the maximisers in Eq.~(\ref{eq:quadratic-model}).  Then
  \begin{equation}
    \label{eq:integrated-step-bounds}
    |\widehat\alpha-\alpha^\star|\leq\min\{1,C(1)/\kappa\},
    \qquad
    0\leq\phi(\alpha^\star)-\phi(\widehat\alpha)
      \leq\frac{C(1)^2}{2\kappa}.
  \end{equation}
  The bounds hold for every global maximiser of $\phi$, including
  endpoints, without assuming concavity.
\end{theorem}

The proof is given in Appendix~\ref{app:integrated-bounds-proof}.
The same accumulated deviation governs scale error linearly and regret
quadratically, relative to the initial negative curvature $\kappa$.
The bound characterises when local utility geometry remains informative
across the shared block's full displacement.  The predictor uses
derivatives at the current recurrent state.  Measurements along the
resulting path quantify its approximation error.
The pointwise location bound in Appendix~\ref{app:curvature-scale-proof}
additionally retains signed cancellation up to the true maximiser.
Appendix~\ref{app:integrated-edge-cases} treats zero deviation and
the sensitivity of the maximising scale when $Q\geq0$.

Finite measurements can bound the accumulated deviation when variation
between samples is controlled.  Proposition~\ref{prop:curvature-partition}
constructs lower and upper sums for $C(1)$ from endpoint curvatures
and valid Lipschitz constants on each partition cell.  Its upper sum
$\overline C$ can replace $C(1)$ in the endpoint, scale, and regret bounds.
A bound on the signed slope residual retains cancellation of curvature
deviations (Appendix~\ref{app:cloud-bounds}).
Cellwise third-derivative bounds supply these constants, and a global
bound $M$ recovers $C(a)\leq Ma^2/2$ under $C^3$ regularity
(Appendix~\ref{app:third-derivative-specialisation}).
The global bound applies worst-case derivative control throughout the
segment.  The integrated quantity accumulates the curvature deviation
along the actual path.  This distinction determines how tightly the
path measurements can bound the error of a local prediction.
These bounds concern real-valued functions.  Their numerical evaluation
requires separate floating-point error accounting.

\subsection{Prediction and Fixed-Direction Intervention}
\label{sec:contraction}

\begingroup
\interlinepenalty=10000
\postdisplaypenalty=10000
After computing the next recurrent state, we evaluate the local model
on the resulting update.
The endpoint comparison measures how closely $A+Q$ predicts
$\Delta U$.  To evaluate the scale prediction, we additionally measure
utility along the segment.
These evaluations are made on a finite grid
$\mathcal A\subseteq[0,1]$ containing $0$ and $1$, with the grid optimum
as the comparison target:
\begin{equation}
  \label{eq:grid-optimum}
  \alpha^\star_{\mathcal{A}}
  :=\min\operatorname*{arg\,max}_{\alpha\in\mathcal{A}}\phi(\alpha).
\end{equation}
\endgroup
The grid target uses sampled utilities, while $\widehat\alpha$ uses
the derivatives at the current recurrent state $H$.
Appendix~\ref{app:grid} gives location and
utility bounds relating this grid target to the continuous optimum
$\alpha^\star$, under their respective path-shape conditions.
We additionally refine the numerical optimum along the same segment,
using the readout's concavity when available and evaluating further
candidate scales for a general readout.
For $A>0,Q<0$, we also evaluate the root prediction $r^{(2)}$ on the
same grid against the interval ending at the first observed negative gain.
This compares the predicted extent of progress with the sampled path.
Locating the continuous boundary $r_1$ additionally requires controlling
the behaviour between grid points.

The curve comparison evaluates a scale predicted separately for each
update.  We additionally measure recovery of finite-step failures under
a common reduction in update magnitude.
The local positivity argument in Appendix~\ref{app:boundary-proof}
guarantees a beneficial scale for each such failure.
We assess recovery at a shared scale by choosing a prespecified fraction
$\alpha_c\in(0,1)$ and
evaluate the readout at $H+\alpha_cD$, holding $H$, $D$, and the
readout fixed.  This is the step-contraction intervention.
A failure is recovered when $\phi(\alpha_c)>0$.
The recovery fraction is the proportion of selected failures satisfying
this condition, and the mean gain averages $\phi(\alpha_c)$ over that
same subset.
When a uniform negative-curvature bound is available, we also select a
scale separately for each update from the sufficient range in
Corollary~\ref{cor:beneficial-step} and evaluate its measured gain.

Recovery quantifies improvement within the subset of updates whose local
direction is favourable and whose full effect is harmful.
To measure the value of intermediate scales over the whole collection of
updates, we compare them with choosing the better of the current and next
states.  This endpoint choice already captures the gain from discarding a
harmful full update.
Let $\mathcal{B}\subseteq[0,1]$ be a finite action set containing
$0$ and $1$, and define
\begin{equation}
  \label{eq:oracle-utilities}
  U_{\mathrm{halt}}:=\max\{U(H),U(H^+)\},
  \qquad
  U_{\mathrm{step}}:=\max_{\alpha\in\mathcal{B}}U(H+\alpha D).
\end{equation}
Both oracles select an action using the measured reference utilities.
The endpoint oracle chooses between $H$ and $H^+$, while the step oracle
can also select an intermediate state.
Since $\mathcal B$ includes both endpoints,
$U_{\mathrm{step}}-U_{\mathrm{halt}}\geq0$, and the endpoint comparison is
proved in Appendix~\ref{app:grid}.
This difference measures the additional utility available from allowing
intermediate scales, beyond the gain already available to endpoint
selection.  We evaluate it over all analysed updates alongside the
conditional recovery measurements.
All comparisons use the same reference utility before further recurrence.

\section{Experiments}
\label{sec:experiments}

We first trace utility across recurrent depth and examine the effect of
one additional loop at fixed transitions (\S~\ref{sec:exp-finite}).
Path measurements then test whether local geometry predicts useful scales
within those updates (\S~\ref{sec:exp-scale}).
Fixed-direction interventions evaluate whether changing the displacement
recovers progress and adds value beyond endpoint selection
(\S~\ref{sec:exp-contraction}).

\subsection{Setup}
\label{sec:exp-setup}

\textbf{Models, tasks, and transitions.}
We study Ouro-1.4B and 2.6B~\citep{zhu_scaling_2026}
alongside Huginn-0125~\citep{NeurIPS2025_70913} to examine the
mechanism across model families and scales, with different recurrent
and readout designs.
Following mathematical evaluations of looped language
models~\citep{yang_stabilizing_2026}, we use the complete
MATH-500~\citep{DBLP:conf/iclr/LightmanKBEBLLS24} test set
($N=500$) and GSM8K~\citep{cobbe2021gsm8k} test set ($N=1319$).
We extend the comparison to commonsense completion using the full
HellaSwag~\citep{ACL2019_53145} validation set ($N=10042$).
All three models are evaluated on each task, giving nine conditions.
For this matrix, transitions are fixed across tasks:
$4\to5$ for both Ouro models and $19\to20$ for Huginn.
A broader depth sweep covers seven labelled splits, including
additional question-answering tasks, giving 21 conditions
(Appendix~\ref{app:exp-coverage}).
Path and quarter-step analyses use MATH-500 test, 256 GSM8K training
questions, and 1000 HellaSwag validation questions, with a
CommonsenseQA~\citep{NAACL2019_35311} training extension of 1024 questions.

\textbf{Evaluation protocol.}
We follow the fixed-direction analysis in \S~\ref{sec:task-utility}
with frozen models. Under teacher forcing, utility is the mean
reference-token log probability for mathematics and the log-softmax
of length-normalised option scores for multiple choice.
Endpoint prediction compares $A$ and $A+Q$ with $\Delta U$ using
sign accuracy, Spearman correlation, and MAE. Scale predictions are
compared with optima on the 21-point grid
$\mathcal A=\{0,0.05,\ldots,1\}$, with regret evaluated at the
nearest grid action to $\widehat\alpha$.
The main matrix uses BF16 recurrent states, FP32 readouts, and FP64
option-score log-softmax. Additional FP64 path analyses retain the
recurrent endpoints and evaluate derivatives and utility with an FP64 readout.
We report 95\% confidence intervals from 2000 paired question-bootstrap
resamples within each analysis population. Utility definitions and
numerical protocols are detailed in Appendices~\ref{app:utilities}
and~\ref{app:exp-protocol}.

\subsection{RQ1: Local Direction versus Finite Progress}
\label{sec:exp-finite}

\textbf{Useful directions within harmful updates.}
Depth sweeps identify utility declines in 19 of 21 model--task conditions
(Appendix~\ref{app:exp-coverage}). To examine how progress is lost within
an additional loop, we compare the initial utility slope with the
full-step gain at the fixed transitions in Table~\ref{tab:mechanism-matrix}.
Finite-step failures as defined in Definition~\ref{def:finite-failure}
occur in all nine conditions, accounting for
1.4--49.2\% of harmful updates: the proposed direction initially improves
reference utility, but its full displacement lowers it.
These failures also occur in all three conditions with positive mean
gain, so the mechanism appears within individual updates even when
the population average improves. A harmful full-step outcome can
therefore conceal a direction along which a shorter move would improve
reference utility. Additional effect-margin checks are reported in
Appendix~\ref{app:exp-protocol}.

\begin{table}[t]
\centering
\caption{\textbf{Finite-step failures and endpoint prediction.} All nine conditions contain finite-step failures. For predicting $\Delta U$, adding $Q$ improves sign accuracy and reduces MAE (bold). Counts are finite-step/harmful updates, and mean gain averages $\Delta U$ over all questions. Rank correlation is Spearman between $A+Q$ and $\Delta U$.}
\label{tab:mechanism-matrix}
\small
\setlength{\tabcolsep}{2pt}
\begin{tabular*}{\linewidth}{@{\extracolsep{\fill}}cccccccccc@{}}
\toprule
\multirow{2}{*}{Model} & \multirow{2}{*}{Task} & Mean gain & \multicolumn{2}{c}{Finite-step failures} & \multicolumn{2}{c}{Sign accuracy (\%) $\uparrow$} & \multirow{2}{*}{\makecell{Rank\\corr. $\uparrow$}} & \multicolumn{2}{c}{MAE ($10^{-4}$) $\downarrow$} \\
\cmidrule(lr){4-5}\cmidrule(lr){6-7}\cmidrule(l){9-10}
 & & $(10^{-3})$ & Count & Share (\%) & $A$ & $A+Q$ & & $A$ & $A+Q$ \\
\midrule
\multirow{3}{*}{\makecell{Ouro\\1.4B}} & MATH-500 & $-8.73$ & 179/364 & 49.18 & 64.20 & \textbf{98.20} & 0.9929 & 139.41 & \textbf{11.99} \\
 & GSM8K & $-14.47$ & 429/1066 & 40.24 & 67.48 & \textbf{98.41} & 0.9985 & 147.51 & \textbf{6.20} \\
 & HellaSwag & $-7.75$ & 261/5771 & 4.52 & 92.49 & \textbf{98.87} & 0.9983 & 73.11 & \textbf{12.45} \\
\addlinespace[4pt]
\multirow{3}{*}{\makecell{Ouro\\2.6B}} & MATH-500 & $-4.40$ & 95/329 & 28.88 & 81.00 & \textbf{98.60} & 0.9973 & 51.12 & \textbf{4.10} \\
 & GSM8K & $-7.77$ & 325/874 & 37.19 & 75.36 & \textbf{97.19} & 0.9952 & 114.61 & \textbf{13.13} \\
 & HellaSwag & $+1.49$ & 88/4719 & 1.86 & 94.24 & \textbf{99.49} & 0.9995 & 29.56 & \textbf{3.05} \\
\addlinespace[4pt]
\multirow{3}{*}{\makecell{Huginn\\3.5B}} & MATH-500 & $+0.75$ & 45/203 & 22.17 & 91.00 & \textbf{99.60} & 0.9998 & 6.81 & \textbf{0.23} \\
 & GSM8K & $-0.48$ & 102/702 & 14.53 & 92.27 & \textbf{99.77} & 0.9998 & 10.39 & \textbf{0.42} \\
 & HellaSwag & $+0.34$ & 65/4695 & 1.38 & 98.58 & \textbf{99.92} & 1.0000 & 2.06 & \textbf{0.17} \\
\bottomrule
\end{tabular*}
\vspace{-4mm}
\end{table}

\textbf{Accumulated curvature and endpoint prediction.}
Proposition~\ref{prop:pathwise-progress} decomposes full-step gain into
the initial slope and an accumulated curvature correction. We test how well local
curvature approximates this correction by comparing $A$ and $A+Q$
with the observed $\Delta U$. Adding $Q$ improves sign accuracy in all
nine matched conditions, reaching 97.2--99.9\%, and reduces MAE by
factors of 5.9--29.8 (Table~\ref{tab:mechanism-matrix}). The largest
sign-accuracy gains occur on Ouro-1.4B's mathematical tasks, where
harmful full updates often have positive initial slopes. These results
support the role of curvature in the mismatch between local improvement
and full-step harm. Additional checks of the affine-readout identity in
Proposition~\ref{prop:linear-endpoint} are reported in
Appendix~\ref{app:exp-kl}.

We also examine whether changes in hidden states and outputs track
task progress. Across the four original intervention sets, state- and
output-change diagnostics have absolute Spearman correlation at most
0.285 with $\Delta U$, while the reference-based prediction $A+Q$
exceeds 0.99 on the same teacher-forced states. These results indicate
that generic change measures provide limited information about the
ranking of utility gains, motivating an analysis of utility along the
update direction. Diagnostic definitions and detailed
comparisons are provided in Appendix~\ref{app:exp-ranking}.

\subsection{RQ2: Predicting the Scale of Progress}
\label{sec:exp-scale}

\textbf{Prediction along the fixed direction.}
Accurate endpoint prediction motivates evaluating scales inside the
update.  On the primary subset $A>0,Q<0$, we compare the quadratic
maximiser with the measured grid optimum, keeping the initial
coefficients fixed.  The first-order rule takes the full step throughout
this subset.  On Ouro-1.4B/MATH-500 and Huginn/GSM8K, predicted and
empirical scales have correlation above 0.995, and the quadratic rule
reduces mean grid regret by more than two orders of magnitude relative to
the first-order rule (Table~\ref{tab:scale-prediction-main}).
Figure~\ref{fig:scale-prediction}(e) compares the quadratic rule with all
five fixed scales.  This grid-regret comparison evaluates
the utility at the predicted scale, complementing the
agreement in scale.  Within the primary subset, it tests whether adding
curvature to the first-order full-step choice improves utility.
The best fixed scale in hindsight differs between these evaluation sets:
$0.5$ for Ouro and $1$ for Huginn.  Even against these separately selected
choices, the quadratic rule reduces mean grid regret by more than two orders
of magnitude.  Its advantage therefore extends beyond replacing the full
step with a uniformly smaller displacement.
Appendix~\ref{app:exp-scale-details} gives the comparison protocol and
results on the additional tasks.

\begin{table}[t]
\centering
\caption{\textbf{Quadratic scale prediction against grid optima.}
Results on Ouro/MATH-500 test and Huginn/GSM8K train256,
restricted to $A>0,Q<0$.
Correlation and MAE compare predicted scales with 21-point grid
optima. Regret is evaluated at the nearest grid action.}
\label{tab:scale-prediction-main}
\small
\setlength{\tabcolsep}{6pt}
\begin{tabular}{@{}lrrrr@{}}
\toprule
Model / task & $n$ & Spearman & Scale MAE & Grid regret \\
\midrule
Ouro-1.4B / MATH-500 & 315 & \textbf{0.9951} & 0.0211 & $\boldsymbol{1.747\times10^{-5}}$ \\
Huginn / GSM8K & 129 & \textbf{0.9998} & 0.0072 & $\boldsymbol{3.482\times10^{-7}}$ \\
\bottomrule
\end{tabular}

\end{table}

\textbf{Intermediate optima and prediction error.}
To test whether local quadratic geometry locates intermediate utility
maxima, we examine paths whose numerical optima lie strictly inside
$[0,1]$. Mean scale error is below $0.019$ across 269 Ouro-1.4B/MATH-500
cases and $0.018$ across 36 Huginn/GSM8K cases
(Table~\ref{tab:fp64-path-prediction}).
These results support the prediction in Eq.~(\ref{eq:predicted-step}):
the initial slope and curvature contain information about how far
to move along the proposed direction
(Figure~\ref{fig:scale-prediction}(a--d)).
Theorem~\ref{thm:integrated-step} relates scale-prediction error to
accumulated curvature variation. We evaluate this bound for both Ouro
models on the complete MATH-500 and GSM8K test sets. Across these four
conditions, the integrated bound $\overline C/\kappa$ has median
$0.058$--$0.114$, compared with $2.84$--$4.61$ for the global bound
$M/(2\kappa)$. It covers the numerical scale-error intervals of all
2034 primary updates and is below the full search range on 99.8\%
of them (Table~\ref{tab:cloud-bound-comparison}).
Thus, accounting for curvature variation along the path gives
informative error bounds for the same local prediction.
Appendix~\ref{app:cloud-bounds} gives the bound distributions and
numerical optimum checks for all four conditions.

\begin{figure}[t]
\centering
\includegraphics[width=\linewidth]{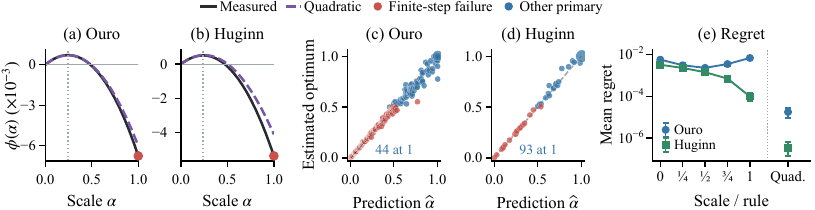}
\caption{\textbf{Scale prediction on Ouro/MATH-500 and Huginn/GSM8K.}
(a,b) Example utility paths and local quadratic approximations.
(c,d) Predicted scales versus numerical optima.
(e) Grid regret for fixed scales and quadratic selection.
Labels in (c,d) count overlapping points at $(1,1)$.}
\label{fig:scale-prediction}
\vspace{-4mm}
\end{figure}

\subsection{RQ3: Recovering Progress by Step Contraction}
\label{sec:exp-contraction}
The preceding analyses in RQ2 show that local geometry predicts both
full-step utility change and useful intermediate scales. Here, we use
fixed-direction interventions to test whether shorter updates recover
the progress suggested by this geometry. Additional comparisons of the
predicted zero crossing with sampled utility paths are reported in
Appendix~\ref{app:exp-boundary}.

\textbf{A common step contraction across questions.}
Finite-step failure guarantees an improving neighbourhood along the
proposed direction, but its extent varies across questions. The improving range in Corollary~\ref{cor:beneficial-step} motivates
step contraction, although its extent depends on the update.
Following the fixed-direction intervention in \S~\ref{sec:contraction},
we first test how often a common scale $\alpha_c=0.25$ recovers
reference utility.
On the four original evaluation sets, this quarter step recovers
72.2--83.2\% of finite-step failures (Figure~\ref{fig:contraction}).
Each recovery gives an intermediate state with higher reference utility
than both endpoints of the original harmful update. These results show
that a single fraction reaches useful intermediate states across most
selected failures in each evaluation set.
Effect margins, task extensions, and transition
selection are detailed in Appendix~\ref{app:exp-protocol}.

\begin{table}[t]
\centering
\caption{\textbf{Scale selection on Ouro-1.4B.} Utility units: $10^{-3}$. Recovery counts show quarter $\to$ selected (failure denominator); utility populations are listed separately. Brackets give paired 95\% CIs. Bound selection uses FP64.}
\label{tab:interventions-main}
\small
\setlength{\tabcolsep}{2.5pt}
\begin{tabular*}{\linewidth}{@{\extracolsep{\fill}}lclrr@{}}
\toprule
Task / rule & \makecell{Recovered\\quarter $\to$ selected} & \makecell[l]{Utility\\population} & \makecell{Mean\\gain} & \makecell{Advantage over quarter\\{}[95\% CI]} \\
\midrule
GSM8K / Quadratic & $306\to \boldsymbol{406}$ (431) & All questions (1319) & \textbf{2.770} & $\boldsymbol{3.610}\;[3.366,3.881]$ \\
MATH-500 / Bound & $132\to \boldsymbol{179}$ (179) & Failures (179) & \textbf{1.074} & $\boldsymbol{0.144}\;[0.061,0.232]$ \\
\bottomrule
\end{tabular*}
\vspace{-4mm}
\end{table}

\textbf{Selecting a scale from local geometry.}
To test whether adapting the scale to each update improves recovery,
we apply the quadratic rule in Eq.~(\ref{eq:quadratic-model}) using
local derivatives of reference utility.
On Ouro-1.4B/GSM8K test at $4\to5$, it recovers 94.2\% of the
431 finite-step failures, compared with 71.0\% for the quarter step.
Across all 1319 questions, the rule also yields positive mean
gain in reference utility and outperforms the quarter step
(Table~\ref{tab:interventions-main}).
The recovery phenomenon also appears on Ouro-2.6B/HellaSwag,
where the quadratic rule recovers 88.0--100\% of finite-step
failures across four separately evaluated transitions from
$1\to2$ through $4\to5$
(Appendix~\ref{app:hellaswag-transitions}).
These results connect the scale prediction examined in RQ2 to
improved recovery along the model's own update directions.
The intervention protocol is detailed in
Appendix~\ref{app:quadratic-intervention}.

\textbf{Selecting a scale from a curvature bound.}
The quadratic rule in Eq.~(\ref{eq:quadratic-model}) estimates where
utility is largest. We next test
whether the sufficient improving range in
Corollary~\ref{cor:beneficial-step} can also guide step selection.
On Ouro-1.4B/MATH-500, we choose
$\alpha_{\mathrm{safe}}=\min\{1,0.9(2A/L_D)\}$ from this range using
the initial slope and the pathwise curvature bound, with the factor
$0.9$ fixed across questions. In the FP64 readout, the selected scales
recover all 179 finite-step failures, compared with 132 for
the quarter step. The mean paired advantage is $1.44\times10^{-4}$
(Table~\ref{tab:interventions-main}). The recovered states have higher
reference utility than both endpoints, linking the sufficient improving
range to progress retained within the original update.

\textbf{Additional value beyond endpoint selection.}
To determine whether contraction offers more than avoiding a harmful
full step, we compare the endpoint and five-scale oracles in
Eq.~(\ref{eq:oracle-utilities}) over all questions in each evaluation set.
Most aggregate gain is available through endpoint selection, while
intermediate scales provide additional reference utility
(Appendix~\ref{app:exp-oracles}). On a finite-step failure, the endpoint
oracle retains the current state, so an interior gain identifies useful
progress that stopping would forgo. Supplementary experiments extend
recovery to CommonsenseQA (Appendix~\ref{app:exp-protocol}) and to
token-level updates within Ouro's four-step inference budget
(Appendix~\ref{app:native-continue}). These observations motivate
using reference-defined geometry as offline supervision for learned
exit and step-scale decisions, a future direction discussed in
Appendix~\ref{app:geometry-adaptive}. The reported gains concern
immediate reference utility after proposal computation and before
further recurrence. Extensions to generated trajectories and
computation costs are discussed in Appendix~\ref{app:limitations}.

\section{Conclusion}
\label{sec:conclusion}

This work characterises finite-step failure in looped Transformers:
an additional recurrent update can lower reference utility despite a
locally improving direction.  For a frozen model and a fixed readout,
accumulated curvature explains this mismatch, while local quadratic
geometry predicts useful step scales and curvature variation bounds
approximation error. Under teacher forcing, we observe recovery of
finite-step failures in both model families when evaluating shorter displacements
along the same direction before further recurrence. A fixed quarter
step recovers 72.2--83.2\% of selected finite-step failures across four
settings. On Ouro-1.4B/GSM8K test, quadratic scale selection raises
recovery from 71.0\% to 94.2\% relative to the quarter step. An observed
decline after another loop can therefore reflect excessive displacement
along a useful direction.  For looped Transformers, this distinction
motivates evaluating how much of each proposed update to apply alongside
how many loops to run, connecting step scale to the task progress
obtained from additional shared computation.

\bibliographystyle{plainnat}
\bibliography{iclr2027_conference}

\appendix

\section{Proofs of the Path Analysis}
\label{app:path-proofs}

We use the state space, utility, displacement, and path defined in
\S~\ref{sec:task-utility}.  In particular, $D=F(H)-H$ is fixed,
$\phi(\alpha)=U(H+\alpha D)-U(H)$, $A=\phi'(0)$, and
$Q=\phi''(0)/2$.  Smoothness is imposed on an open neighbourhood of the
tested segment, and endpoint derivatives use the resulting real-valued
continuation.

\subsection{Directional Derivatives}
\label{app:directional-derivatives}

\begin{proof}[Derivation of Eq.~(\ref{eq:directional-coefficients})]
  The affine path has derivative $\gamma'(\alpha)=D$.  Composing the
  differential of $U$ with this path and subtracting the constant $U(H)$
  gives $\phi'(\alpha)=dU_{\gamma(\alpha)}[D]$.
  In the inner product space, the differential is represented by the
  gradient.  Differentiating its inner product with the fixed vector $D$
  then gives
  \begin{equation}
    \phi'(\alpha)
      =\bigl\langle\nabla U(\gamma(\alpha)),D\bigr\rangle_{\mathcal{H}},
    \qquad
    \phi''(\alpha)
      =\bigl\langle D,\nabla^2 U(\gamma(\alpha))[D]\bigr\rangle_{\mathcal{H}}.
  \end{equation}
  Evaluation at zero gives $A$ and $Q$.  The identities involve the fixed
  vector $D$ and the derivatives of $U$ at points of the segment.
\end{proof}

\subsection{Finite Progress and Beneficial Scales}
\label{app:progress-proof}

\begin{lemma}[Increment Comparison]
  \label{lem:increment-comparison}
  Let $f,g$ be continuously differentiable on an open neighbourhood of
  $[0,1]$.  If $f'(s)\leq g'(s)$ for every $s\in[0,1]$, then
  $f(\alpha)-f(0)\leq g(\alpha)-g(0)$ for every $\alpha\in[0,1]$.
\end{lemma}

\begin{proof}
  Continuity makes both derivatives integrable on $[0,\alpha]$.
  The fundamental theorem of calculus and monotonicity of the integral give
  \begin{equation}
    f(\alpha)-f(0)=\int_0^\alpha f'(s)\,ds
      \leq\int_0^\alpha g'(s)\,ds=g(\alpha)-g(0).
  \end{equation}
\end{proof}

\begin{proof}[Proof of Proposition~\ref{prop:pathwise-progress}]
  Fix $\alpha\in[0,1]$ and define
  $g_\alpha(s):=\phi(s)+(\alpha-s)\phi'(s)$.
  The product rule gives $g_\alpha'(s)=(\alpha-s)\phi''(s)$.
  This derivative is continuous on $[0,\alpha]$, so the fundamental theorem
  of calculus yields
  \begin{equation}
    \int_0^\alpha(\alpha-s)\phi''(s)\,ds
      =g_\alpha(\alpha)-g_\alpha(0)
      =\phi(\alpha)-\alpha A.
  \end{equation}
  The last equality uses $\phi(0)=0$ and $\phi'(0)=A$.
  At a finite-step failure, $\phi(1)<0$ and $A>0$, so
  $\int_0^1(1-s)\phi''(s)\,ds=\phi(1)-A<-A<0$.
\end{proof}

\begin{proof}[Proof of Corollary~\ref{cor:beneficial-step}]
  First compare the functions $s\mapsto-L_Ds$ and $s\mapsto\phi'(s)$.
  Their derivatives satisfy $-L_D\leq\phi''(s)$, so
  Lemma~\ref{lem:increment-comparison} gives
  $\phi'(s)\geq A-L_Ds$ on $[0,1]$.
  Next define $b(s):=As-L_Ds^2/2$, whose derivative is $A-L_Ds$.
  Applying the same comparison to $b$ and $\phi$, with
  $b(0)=\phi(0)=0$, gives
  \begin{equation}
    \phi(\alpha)
      \geq\alpha A-\frac{L_D}{2}\alpha^2
      =\alpha\left(A-\frac{L_D}{2}\alpha\right).
  \end{equation}
  For $L_D>0$, this expression is positive whenever
  $0<\alpha<2A/L_D$.  For $L_D=0$, it equals $\alpha A>0$ at every
  positive scale in the interval.
\end{proof}

\subsection{The Continuous Progress Boundary}
\label{app:boundary-proof}

\begin{proof}[Justification of $0<r_1<1$ for a finite-step failure]
  Differentiability at zero, $\phi(0)=0$, and $A>0$ imply
  $\phi(\alpha)/\alpha\to A$ as $\alpha\downarrow0$.
  Hence $\phi(\alpha)>0$ on some interval $(0,\varepsilon)$.
  Set $r_0=\min\{\varepsilon/2,1/2\}$.  Then $r_0$ belongs to the
  prefix set in Eq.~(\ref{eq:progress-boundary}).  This set is nonempty
  and bounded above by one, so its supremum satisfies $r_1\geq r_0>0$.
  Since $\phi(1)<0$, continuity at one gives $\eta>0$ such that
  $|\alpha-1|<\eta$ implies $\phi(\alpha)<0$.
  Put $b=1-\min\{\eta/2,1/2\}$.  Then $0\leq b<1$ and
  $\phi(b)<0$.  Every member of the prefix set is strictly less than
  $b$, making $b$ an upper bound.  Hence $r_1\leq b<1$.
\end{proof}

The local positive interval requires only the derivative at zero.  The
curvature bound in Corollary~\ref{cor:beneficial-step} additionally gives an
explicit sufficient range of scales.  A two-sided curvature bound also
places the boundary and the quadratic root prediction in a common
interval.

\begin{proposition}[Root Prediction under Two-Sided Curvature Bounds]
  \label{prop:root-sandwich}
  Under Proposition~\ref{prop:pathwise-progress}, suppose $A>0$ and
  $-L_D\leq\phi''(s)\leq-\mu$ on $[0,1]$ for constants
  $0<\mu\leq L_D$.  Then $\kappa=-\phi''(0)\in[\mu,L_D]$, so
  $r^{(2)}=-A/Q=2A/\kappa\in[2A/L_D,\,2A/\mu]$, and
  $r_1\geq\min\{1,2A/L_D\}$.  If moreover $2A/\mu<1$, then
  $r_1\leq2A/\mu$, and consequently
  \begin{equation}
    \label{eq:root-sandwich}
    |r^{(2)}-r_1|\leq2A\left(\frac{1}{\mu}-\frac{1}{L_D}\right).
  \end{equation}
\end{proposition}

\begin{proof}
  The membership of $\kappa$ and hence of $r^{(2)}$ follows from the
  curvature bounds at $s=0$.
  For the lower bound on $r_1$, Corollary~\ref{cor:beneficial-step}
  gives $\phi(\alpha)\geq\alpha(A-L_D\alpha/2)>0$ for
  $0<\alpha<2A/L_D$.  With $\phi(0)=0$, every
  $r<\min\{1,2A/L_D\}$ belongs to the prefix set in
  Eq.~(\ref{eq:progress-boundary}), so its supremum is at least
  $\min\{1,2A/L_D\}$.
  For the upper bound, apply Lemma~\ref{lem:increment-comparison} to
  $\phi'$ and $s\mapsto A-\mu s$, whose derivatives satisfy
  $\phi''\leq-\mu$.  This gives $\phi'(s)\leq A-\mu s$ on $[0,1]$.
  Applying the same comparison to $\phi$ and $b(s)=As-\mu s^2/2$ gives
  $\phi(\alpha)\leq\alpha(A-\mu\alpha/2)$, which is negative for
  $\alpha>2A/\mu$.  If $2A/\mu<1$, then $\phi<0$ on $(2A/\mu,1]$, so no
  $r>2A/\mu$ belongs to the prefix set and $r_1\leq2A/\mu$.
  Both $r_1$ and $r^{(2)}$ then lie in $[2A/L_D,2A/\mu]$, whose length
  is the right-hand side of Eq.~(\ref{eq:root-sandwich}).
\end{proof}

In the affine-logit case, $\phi''$ is the negative mean variance in
Eq.~(\ref{eq:linear-readout-curvature}), so $\mu$ and $L_D$ bound the
minimum and maximum of that variance along the segment.  The
experiments evaluate only the upper envelope $L_D$ and compare
$r^{(2)}$ with the sampled crossings
(Appendix~\ref{app:exp-boundary}).  No bound of the form
Eq.~(\ref{eq:root-sandwich}) is computed.

\subsection{Constrained Quadratic Maximisation}
\label{app:quadratic-max}

\begin{proof}[Existence and tie convention in Eq.~(\ref{eq:quadratic-model})]
  Let $f$ be either $\phi$ or $q$.  Compactness of $[0,1]$ and
  continuity of $f$ give a maximiser $b$.  The set
  $S=[0,1]\cap f^{-1}(\{f(b)\})$ is a closed subset of a compact
  interval and contains $b$.  Minimising the identity function on $S$
  gives a point $a\in S$ no larger than any other member.
  Since $f(a)=f(b)$, this point maximises $f$ on $[0,1]$ and is its
  smallest maximiser.
  If $a$ and $\widetilde a$ both have this property, maximality gives
  $f(a)=f(\widetilde a)$, and their respective minimality gives
  $a\leq\widetilde a$ and $\widetilde a\leq a$.  Thus they coincide.
\end{proof}

\begin{proof}[Derivation of Eq.~(\ref{eq:predicted-step})]
  Suppose $Q<0$.  For any candidate $c\in[0,1]$,
  \begin{equation}
    q(x)-q(c)=(A+2Qc)(x-c)+Q(x-c)^2.
  \end{equation}
  If $(A+2Qc)(x-c)\leq0$ for every $x\in[0,1]$, this identity makes
  $c$ a maximiser, and the strict inequality $Q(x-c)^2<0$ for $x\neq c$
  makes it unique.
  Write $u=-A/(2Q)$, so $A+2Qu=0$.
  If $u\leq0$, then $A\leq0$ and $c=0$ satisfies the required
  inequality.  If $u>0$ and $u\geq1$, then $A+2Q\geq0$ and $c=1$
  satisfies it.  In the remaining case $0<u<1$, take $c=u$, where
  the derivative vanishes.  These branches give the clipped formula.

  For $Q\geq0$, the inequality $x^2\leq x$ on $[0,1]$ gives
  $q(x)\leq x(A+Q)$.
  If $A+Q>0$, this is at most $A+Q=q(1)$, with strict inequality
  when $x<1$.  If $A+Q\leq0$, it is at most zero, and $c=0$ is the
  smallest maximiser.  This includes $Q=0$ and the constant-zero quadratic.
  For $A>0,Q<0$, the factorisation
  $q(\alpha)=\alpha(A+Q\alpha)$ also gives the positive root
  $r^{(2)}=-A/Q=2(-A/(2Q))$.
\end{proof}

\subsection{Quadratic Approximation Error}
\label{app:quadratic-proof}

\begin{lemma}[Power Increment Bound]
  \label{lem:power-increment}
  Let $f$ be continuously differentiable on an open neighbourhood of
  $[0,1]$, and suppose $|f'(s)|\leq Bs^n$ on $[0,1]$ for
  $B\in\mathbb R$ and an integer $n\geq0$.  Then
  \begin{equation}
    |f(\alpha)-f(0)|\leq\frac{B}{n+1}\alpha^{n+1},
    \qquad \alpha\in[0,1].
  \end{equation}
\end{lemma}

\begin{proof}
  The derivative bounds are $-Bs^n\leq f'(s)\leq Bs^n$.
  The two bounding polynomials
  $g_\pm(s)=\pm B s^{n+1}/(n+1)$ have derivatives $\pm Bs^n$
  and vanish at zero.  Applying Lemma~\ref{lem:increment-comparison}
  first to $f,g_+$ and then to $g_-,f$ gives the two sides of the
  absolute-value inequality.
\end{proof}

\subsubsection{Endpoint Prediction Error}
\label{app:endpoint-prediction-proof}

\begin{proof}[Derivation of Eq.~(\ref{eq:endpoint-prediction-bound})]
  Under Assumption~\ref{ass:readout-regularity}, the fundamental theorem
  of calculus gives
  $\phi'(a)-q'(a)=\int_0^a[\phi''(s)-\phi''(0)]\,ds$,
  whose absolute value is at most $C(a)\leq C(1)$ on $[0,1]$.
  Applying the mean value inequality to $\phi-q$ between zero and one,
  with $\phi(0)=q(0)=0$, yields
  $|\Delta U-(A+Q)|\leq C(1)$.
  If $A+Q>C(1)$, then $\Delta U\geq A+Q-C(1)>0$, and
  if $A+Q<-C(1)$, then $\Delta U\leq A+Q+C(1)<0$.
  This proves the sign statement without a sign restriction on $A$ or $Q$.
\end{proof}

\subsubsection{Integrated Curvature Deviation}
\label{app:curvature-scale-proof}

\begin{proposition}[Scale Error from Integrated Curvature Deviation]
  \label{prop:curvature-scale-error}
  Let $\phi\in C^2(\mathcal O)$ on an open neighbourhood $\mathcal O$
  of $[0,1]$, and suppose $\phi''(0)<0$.  Set
  $\kappa:=-\phi''(0)=-2Q>0$, and let $\alpha^\star$ and
  $\widehat\alpha$ be as in Eq.~(\ref{eq:quadratic-model}).  Then
  \begin{equation}
    \label{eq:curvature-scale-error}
    |\widehat\alpha-\alpha^\star|
    \leq\frac{\left|\int_0^{\alpha^\star}
                    [\phi''(s)-\phi''(0)]\,ds\right|}{\kappa}
    \leq\frac{C(\alpha^\star)}{\kappa}
    \leq\frac{C(1)}{\kappa}.
  \end{equation}
  The result includes endpoint maximisers and does not require concavity
  of $\phi$.
\end{proposition}

\begin{proof}[Proof of Proposition~\ref{prop:curvature-scale-error}]
  Continuity on the compact interval ensures existence of both maximisers.
  The argument applies to any global maximiser $\alpha^\star$ of $\phi$, and
  the smallest-maximiser convention merely fixes its name.
  Since $\kappa=-\phi''(0)>0$, the quadratic
  $q(\alpha)=\phi'(0)\alpha-\kappa\alpha^2/2$ is strictly concave,
  so its constrained maximiser $\widehat\alpha$ is unique.

  First, let $f$ be differentiable at a constrained maximiser
  $a\in[0,1]$.  For any $b\in[0,1]$, the curve
  $t\mapsto a+t(b-a)$ remains feasible for $t\in[0,1]$.
  Optimality implies
  $f(a+t(b-a))-f(a)\leq0$ for $t>0$.
  Dividing by $t$ and taking $t\downarrow0$ gives
  \begin{equation}
    \label{eq:app-constrained-direction}
    f'(a)(b-a)\leq0.
  \end{equation}
  This reasoning uses only feasible one-sided displacements, so it also
  applies when $a=0$ or $a=1$.  It does not assert that $f'(a)=0$ there.

  Put $\delta:=\widehat\alpha-\alpha^\star$.
  Applying Eq.~(\ref{eq:app-constrained-direction}) first to $\phi$
  and then to $q$ yields
  \begin{equation}
    \phi'(\alpha^\star)\delta\leq0,
    \qquad q'(\widehat\alpha)\delta\geq0.
  \end{equation}
  Since $q'(\alpha)=\phi'(0)-\kappa\alpha$, we have
  $q'(\alpha^\star)-q'(\widehat\alpha)=\kappa\delta$.
  Therefore
  \begin{equation}
    \begin{aligned}
      \kappa\delta^2
      &=\bigl(q'(\alpha^\star)-q'(\widehat\alpha)\bigr)\delta\\
      &\leq q'(\alpha^\star)\delta\\
      &\leq\bigl(q'(\alpha^\star)-\phi'(\alpha^\star)\bigr)\delta\\
      &\leq|q'(\alpha^\star)-\phi'(\alpha^\star)|\,|\delta|.
    \end{aligned}
  \end{equation}
  For $\delta\neq0$, division by $\kappa|\delta|>0$ gives
  \begin{equation}
    \label{eq:app-derivative-mismatch-distance}
    |\delta|\leq
      \frac{|\phi'(\alpha^\star)-q'(\alpha^\star)|}{\kappa}.
  \end{equation}
  For $\delta=0$, the same inequality holds because its right-hand side
  is nonnegative.

  Next, $q'(a)=\phi'(0)+\phi''(0)a$ and the fundamental theorem of
  calculus give, for every $a\in[0,1]$,
  \begin{equation}
    \label{eq:app-curvature-mismatch-integral}
    \phi'(a)-q'(a)
      =\phi'(a)-\phi'(0)-a\phi''(0)
      =\int_0^a[\phi''(s)-\phi''(0)]\,ds.
  \end{equation}
  All integrands are continuous on the compact segment, hence integrable.
  Taking absolute values gives
  \begin{equation}
    \left|\int_0^a[\phi''(s)-\phi''(0)]\,ds\right|
      \leq\int_0^a|\phi''(s)-\phi''(0)|\,ds=C(a).
  \end{equation}
  For $0\leq a\leq b\leq1$,
  $C(b)-C(a)=\int_a^b|\phi''(s)-\phi''(0)|\,ds\geq0$.
  Thus $C(\alpha^\star)\leq C(1)$.
  Substitution into Eq.~(\ref{eq:app-derivative-mismatch-distance})
  proves the complete chain in Eq.~(\ref{eq:curvature-scale-error}).
  Neither this proof nor the distance conclusion uses $\phi(0)=0$,
  $\phi'(0)>0$, or concavity of the actual path.
\end{proof}

\subsubsection{Integrated Location and Regret Bounds}
\label{app:integrated-bounds-proof}

\begin{proof}[Proof of Theorem~\ref{thm:integrated-step}]
  Continuity on $[0,1]$ gives a global
  maximiser, and strict concavity of $q$ gives its unique maximiser
  $\widehat\alpha$.  Fix any global maximiser $\alpha^\star$ of
  $\phi$.  Define the derivative residual and value residual by
  $e(a):=\phi'(a)-q'(a)$ and $\mathcal R(a):=\phi(a)-q(a)$.
  Eq.~(\ref{eq:app-curvature-mismatch-integral}) and monotonicity of
  $C$ give $|e(a)|\leq C(a)\leq C(1)$ on $[0,1]$.
  Put $\varepsilon:=C(1)\geq0$.

  Proposition~\ref{prop:curvature-scale-error} bounds the location
  error by $\varepsilon/\kappa$.  Both maximisers belong to $[0,1]$,
  so their distance is also at most one.  This proves the first
  conclusion of Eq.~(\ref{eq:integrated-step-bounds}).

  For the value bound, $\mathcal R'=e$ and the derivative bound on the
  convex interval imply, by the mean value inequality,
  \begin{equation}
    \label{eq:residual-increment}
    |\mathcal R(v)-\mathcal R(u)|\leq\varepsilon|v-u|,
    \qquad u,v\in[0,1].
  \end{equation}
  This step uses differentiability and the uniform derivative bound.
  It applies in either order of $u,v$.
  The exact quadratic expansion at its constrained maximiser gives
  \begin{equation}
    q(a)-q(\widehat\alpha)
    =q'(\widehat\alpha)(a-\widehat\alpha)
      -\frac{\kappa}{2}(a-\widehat\alpha)^2
    \leq-\frac{\kappa}{2}(a-\widehat\alpha)^2.
  \end{equation}
  The inequality uses the feasible first-order condition
  in Eq.~(\ref{eq:app-constrained-direction}).
  With $d:=|\alpha^\star-\widehat\alpha|$, decompose the utility
  difference into the residual increment and the quadratic gap:
  \begin{equation}
    \label{eq:residual-regret-distance}
    \begin{aligned}
      0\leq\phi(\alpha^\star)-\phi(\widehat\alpha)
      &\leq\varepsilon d-\frac{\kappa}{2}d^2\\
      &=\frac{\varepsilon^2}{2\kappa}
        -\frac{\kappa}{2}(d-\varepsilon/\kappa)^2
      \leq\frac{\varepsilon^2}{2\kappa}.
    \end{aligned}
  \end{equation}
  The lower bound follows from maximality of $\alpha^\star$.
  Substitution of $\varepsilon=C(1)$ proves the result.  No value of
  $\phi(0)$ enters the residual increment, so adding a constant to the
  utility leaves the argument unchanged.
\end{proof}

\begin{corollary}[Refinement by the Scale Domain]
  \label{cor:integrated-domain-regret}
  Under Theorem~\ref{thm:integrated-step},
  \begin{equation}
    \phi(\alpha^\star)-\phi(\widehat\alpha)\leq
    \begin{cases}
      C(1)^2/(2\kappa),& C(1)\leq\kappa,\\
      C(1)-\kappa/2,& C(1)>\kappa.
    \end{cases}
  \end{equation}
  The same conclusion holds with any upper bound $\overline C\geq C(1)$
  in place of $C(1)$.
\end{corollary}
\begin{proof}
  The distance $d$ in Eq.~(\ref{eq:residual-regret-distance}) belongs
  to $[0,1]$.  If $\varepsilon\leq\kappa$, the quadratic in $d$
  has its maximum at $\varepsilon/\kappa$.  If
  $\varepsilon>\kappa$, its maximum on $[0,1]$ is
  $\varepsilon-\kappa/2$ at one.  For an upper bound $\overline C$,
  the same proof starts with $|e(a)|\leq\overline C$, and
  nonnegativity follows from $0\leq C(1)\leq\overline C$.
\end{proof}

\subsubsection{Finite-Partition Enclosures}
\label{app:curvature-enclosure-proof}

To connect this integral to finite measurements, let
$g(s):=\phi''(s)-\phi''(0)$ and partition the interval as
$0=t_0<\cdots<t_N=1$, with $N\geq1$.
For each cell $\ell=0,\ldots,N-1$, suppose
$|g(u)-g(v)|\leq L_\ell|u-v|$ for all $u,v\in[t_\ell,t_{\ell+1}]$,
where $L_\ell\geq0$.  Define
\begin{equation}
  \label{eq:curvature-cell-terms}
  h_\ell:=t_{\ell+1}-t_\ell,
  \quad T_\ell:=\frac{h_\ell}{2}(|g(t_\ell)|+|g(t_{\ell+1})|),
  \quad P_\ell:=\frac{L_\ell h_\ell^2}{4}.
\end{equation}
\begin{proposition}[Finite-Partition Curvature Bounds]
  \label{prop:curvature-partition}
  Under the assumptions of Theorem~\ref{thm:integrated-step} and
  these cellwise conditions,
  \begin{equation}
    \label{eq:curvature-enclosure}
    \underline C:=\sum_{\ell=0}^{N-1}\max\{0,T_\ell-P_\ell\}
    \leq C(1)\leq
    \overline C:=\sum_{\ell=0}^{N-1}(T_\ell+P_\ell).
  \end{equation}
  Both bounds in Theorem~\ref{thm:integrated-step} remain valid when
  $C(1)$ is replaced by $\overline C$.
\end{proposition}

\begin{lemma}[Nearest-Endpoint Integral Bound]
  \label{lem:nearest-endpoint-enclosure}
  Let $l<r$ and let $g$ be $L$-Lipschitz on $[l,r]$, with $L\geq0$.
  For $h=r-l$, $T=h(|g(l)|+|g(r)|)/2$, and $P=Lh^2/4$,
  \begin{equation}
    \left|\int_l^r |g(s)|\,ds-T\right|\leq P,
    \qquad
    \max\{0,T-P\}\leq\int_l^r |g(s)|\,ds\leq T+P.
  \end{equation}
\end{lemma}
\begin{proof}
  Set $m=(l+r)/2$.  Lipschitz continuity and the reverse triangle
  inequality give $\big||g(s)|-|g(l)|\big|\leq L(s-l)$ on $[l,m]$
  and $\big||g(s)|-|g(r)|\big|\leq L(r-s)$ on $[m,r]$.
  Integrating each half-interval and using the triangle inequality gives
  \begin{equation}
    \left|\int_l^r |g(s)|\,ds-T\right|
      \leq \frac{L}{2}(m-l)^2+\frac{L}{2}(r-m)^2
      =\frac{Lh^2}{4}.
  \end{equation}
  The constant endpoint contributions sum to $T$.
  The absolute-value bound and nonnegativity of the integral give the
  two-sided enclosure.  The proof requires no derivative of $|g|$.
\end{proof}

\begin{proof}[Proof of Proposition~\ref{prop:curvature-partition}]
  Apply Lemma~\ref{lem:nearest-endpoint-enclosure} on each cell, then
  sum the inequalities.  Strictly increasing endpoints and additivity
  of adjacent integrals give
  \begin{equation}
    C(1)=\sum_{\ell=0}^{N-1}
      \int_{t_\ell}^{t_{\ell+1}}|g(s)|\,ds.
  \end{equation}
  This proves Eq.~(\ref{eq:curvature-enclosure}).  The residual bound
  $|e(a)|\leq C(1)\leq\overline C$ feeds directly into the location
  and regret arguments of Theorem~\ref{thm:integrated-step}.

  For the third-derivative bound, suppose $\phi$ is $C^3$ on an
  open neighbourhood of $[0,1]$ and
  $|\phi'''(s)|\leq L_\ell$ at every point of each cell.
  Since $g'=\phi'''$, the mean value inequality makes $g$
  $L_\ell$-Lipschitz there.  Every cell is contained in $[0,1]$,
  so these bounds supply exactly the hypotheses just used.
\end{proof}

The finite sum specifies an upper bound when its cellwise constants
are valid throughout their intervals.  A maximum over finitely many
sampled third derivatives does not supply this condition.  These bounds
hold in exact arithmetic.  Floating-point derivative evaluation and
outward rounding require separate numerical error control.

\subsubsection{Zero Deviation and Nonnegative Initial Curvature}
\label{app:integrated-edge-cases}

\begin{corollary}[Zero Deviation and Multiple Maximisers]
  Under the $C^2$ setup, $C(1)=0$ implies
  $\phi(a)=q(a)+\phi(0)$ on $[0,1]$.
  If additionally $Q<0$, every true maximiser then equals $\widehat\alpha$.
  For any value of $C(1)$, under $Q<0$ and $\kappa=-2Q>0$, any two
  global maximisers $a,b\in[0,1]$ satisfy $|a-b|\leq2C(1)/\kappa$.
\end{corollary}
\begin{proof}
  If $C(1)=0$, the residual derivative satisfies $|e(a)|\leq0$.
  The mean value inequality for $\phi-q$, with derivative zero,
  makes the value residual constant, proving the path identity.  The zero-distance
  conclusion follows from Theorem~\ref{thm:integrated-step}.
  Applying its distance bound to each maximiser and using the triangle
  inequality through $\widehat\alpha$ gives the diameter bound.
\end{proof}

\begin{proposition}[Location Instability for $Q\geq0$]
  For every $Q\geq0$ and $\eta>0$, the path
  $\phi(a)=Qa(a-1)+\eta a^3$ has Taylor quadratic $q(a)=Qa(a-1)$.
  Its smallest quadratic maximiser is zero and its unique true maximiser
  is one, while $|\phi'(a)-q'(a)|\leq3\eta$ on $[0,1]$.
\end{proposition}
\begin{proof}
  Differentiation gives $\phi'(0)=-Q$ and $\phi''(0)=2Q$,
  identifying the Taylor model.  Since $a^2\leq a$ on $[0,1]$,
  $q(a)\leq0=q(0)=q(1)$, so its smallest maximiser is zero.
  Moreover $\phi(a)\leq\eta a^3\leq\eta a\leq\eta=\phi(1)$, and
  equality at a maximum forces $a=1$ because $\eta>0$.
  Finally $\phi'(a)-q'(a)=3\eta a^2\in[0,3\eta]$.
\end{proof}
Thus an arbitrarily small derivative residual does not imply a small
location error when the initial quadratic curvature is nonnegative.

\begin{proposition}[Value Bound for an Arbitrary Quadratic]
  For any $A,Q$, suppose $\phi$ is differentiable on a neighbourhood
  of $[0,1]$ and $|\phi'(a)-(A+2Qa)|\leq\varepsilon$ for every
  $a\in[0,1]$, with $\varepsilon\geq0$.
  If $\alpha^\star$ maximises $\phi$ and $\widehat\alpha$ maximises
  $q(a)=Aa+Qa^2$, both on $[0,1]$, then
  $0\leq\phi(\alpha^\star)-\phi(\widehat\alpha)\leq\varepsilon$.
\end{proposition}
\begin{proof}
  The residual increment is at most
  $\varepsilon|\alpha^\star-\widehat\alpha|$ by the mean value
  inequality, and the quadratic gap is nonpositive by maximality.
  Their sum is bounded by $\varepsilon$ because both scales lie in
  $[0,1]$.  Maximality of $\alpha^\star$ gives nonnegativity.
\end{proof}

\begin{corollary}[Combined Integrated and Taylor Regret Bounds]
  \label{cor:combined-regret}
  Under Theorem~\ref{thm:integrated-step} and
  Assumption~\ref{ass:curvature-variation},
  \begin{equation}
    \begin{aligned}
      \phi(\alpha^\star)-\phi(\widehat\alpha)
      &\leq\min\left\{\frac{M}{6}
          \bigl((\alpha^\star)^3+\widehat\alpha^3\bigr),
          \frac{C(1)^2}{2\kappa}\right\}\\
      &\leq\min\left\{\frac{M}{3},\frac{M^2}{8\kappa}\right\}.
    \end{aligned}
  \end{equation}
\end{corollary}
\begin{proof}
  Apply the Taylor result in Theorem~\ref{thm:quadratic-step} to the
  centred path $\phi-\phi(0)$.  It has the same derivatives and
  maximisers, and centring cancels from its regret.  Combining that bound
  with Theorem~\ref{thm:integrated-step} gives the first minimum.
  Corollary~\ref{cor:curvature-third-derivative} gives $C(1)\leq M/2$, and
  the two scales are in $[0,1]$, giving the second minimum.
\end{proof}

\begin{corollary}[Third-Derivative Specialisation]
  \label{cor:curvature-third-derivative}
  Under Proposition~\ref{prop:curvature-scale-error}, suppose additionally
  that $\phi\in C^3(\mathcal O)$ and $|\phi'''(s)|\leq M$ on $[0,1]$
  for $M\geq0$.  Then, for every $a\in[0,1]$,
  \begin{equation}
    C(a)\leq\frac{Ma^2}{2},
    \qquad
    |\widehat\alpha-\alpha^\star|
      \leq\frac{M(\alpha^\star)^2}{2\kappa}
      \leq\frac{M}{2\kappa}.
  \end{equation}
\end{corollary}

\begin{proof}
  A further application of the fundamental theorem of calculus gives
  \begin{equation}
    |\phi''(s)-\phi''(0)|
      =\left|\int_0^s\phi'''(u)\,du\right|
      \leq\int_0^s|\phi'''(u)|\,du\leq Ms.
  \end{equation}
  Integrating this inequality from zero to $a$ yields
  $C(a)\leq\int_0^a Ms\,ds=Ma^2/2$.
  Proposition~\ref{prop:curvature-scale-error}, $\kappa>0$, and
  $0\leq\alpha^\star\leq1$ give the asserted distance bound.
\end{proof}

\subsubsection{Utility Regret and the Third-Derivative Bound}
\label{app:third-derivative-specialisation}

\begin{assumption}[Directional Curvature Variation]
  \label{ass:curvature-variation}
  Let $\mathcal O\subseteq\mathbb R$ be open with
  $[0,1]\subseteq\mathcal O$.  The utility gain along the path satisfies
  $\phi\in C^3(\mathcal O)$, and there is a constant $M\geq0$ such that
  \begin{equation}
    \label{eq:third-derivative-bound}
    |\phi'''(\alpha)|\leq M,
    \qquad \alpha\in[0,1].
  \end{equation}
\end{assumption}

\begin{theorem}[Quadratic Step-Scale Approximation]
  \label{thm:quadratic-step}
  Under Assumption~\ref{ass:curvature-variation}, suppose $\phi(0)=0$.
  With the quantities in
  Eqs.~(\ref{eq:directional-coefficients}) and~(\ref{eq:quadratic-model}),
  \begin{equation}
    \label{eq:step-regret}
    0\leq\phi(\alpha^\star)-\phi(\widehat\alpha)
    \leq\frac{M}{6}\bigl((\alpha^\star)^3+\widehat\alpha^3\bigr)
    \leq\frac{M}{3}.
  \end{equation}
  If additionally $Q<0$, let $\kappa:=-2Q>0$.  Then
  \begin{equation}
    \label{eq:step-distance}
    |\widehat\alpha-\alpha^\star|
    \leq\frac{M}{2\kappa}(\alpha^\star)^2
    \leq\frac{M}{2\kappa}.
  \end{equation}
\end{theorem}

\begin{proof}[Proof of Theorem~\ref{thm:quadratic-step}]
  Under Assumption~\ref{ass:curvature-variation}, apply
  Lemma~\ref{lem:power-increment} with $n=0$ to $\phi''$,
  whose derivative is bounded by $M$.  This gives
  $|\phi''(\alpha)-\phi''(0)|\leq M\alpha$.
  The function $f_1(\alpha)=\phi'(\alpha)-\phi''(0)\alpha$ has
  derivative $\phi''(\alpha)-\phi''(0)$.  Applying the same lemma
  with $n=1$ yields
  \begin{equation}
    |\phi'(\alpha)-(A+\phi''(0)\alpha)|\leq\frac{M}{2}\alpha^2.
  \end{equation}
  Next take $f_2(\alpha)=\phi(\alpha)-q(\alpha)$.
  Its derivative is the preceding error, and $f_2(0)=0$.
  The lemma with $n=2$ and coefficient $M/2$ gives the function-value
  bound.  Using $\phi''(0)=2Q$, the two bounds are
  \begin{equation}
    \label{eq:app-taylor-remainders}
    |\phi(\alpha)-q(\alpha)|\leq\frac{M}{6}\alpha^3,
    \qquad
    |\phi'(\alpha)-q'(\alpha)|\leq\frac{M}{2}\alpha^2.
  \end{equation}
  Optimality of $\alpha^\star$ gives the nonnegative lower bound in
  Eq.~(\ref{eq:step-regret}).  Optimality of $\widehat\alpha$ for $q$
  gives
  \begin{equation}
    \begin{aligned}
      \phi(\alpha^\star)-\phi(\widehat\alpha)
      &=\bigl[\phi(\alpha^\star)-q(\alpha^\star)\bigr]
        +\bigl[q(\alpha^\star)-q(\widehat\alpha)\bigr]
        +\bigl[q(\widehat\alpha)-\phi(\widehat\alpha)\bigr]\\[2pt]
      &\leq\frac{M}{6}
         \bigl((\alpha^\star)^3+\widehat\alpha^3\bigr)
      \leq\frac{M}{3}.
    \end{aligned}
  \end{equation}

  For the location bound, $Q<0$ gives $\kappa=-2Q=-\phi''(0)>0$.
  Proposition~\ref{prop:curvature-scale-error} and
  Corollary~\ref{cor:curvature-third-derivative} give directly
  \begin{equation}
    |\widehat\alpha-\alpha^\star|
      \leq\frac{C(\alpha^\star)}{\kappa}
      \leq\frac{M(\alpha^\star)^2}{2\kappa}
      \leq\frac{M}{2\kappa}.
  \end{equation}
  This includes boundary maxima and requires no concavity of $\phi$.
\end{proof}

\section{Finite Action Sets}
\label{app:grid}

Let $\mathcal{A}=\{a_0,\ldots,a_m\}$, with $m\geq1$, be ordered with
$0=a_0<a_1<\cdots<a_m=1$, and let
$h_{\mathcal{A}}:=\max_{\ell=0,\ldots,m-1}(a_{\ell+1}-a_\ell)$.
The sampled optimum $\alpha^\star_{\mathcal{A}}$ is defined in
Eq.~(\ref{eq:grid-optimum}).

\begin{proposition}[Grid Resolution]
  \label{prop:grid-resolution}
  If $\phi$ is continuous and concave on $[0,1]$ with a unique maximiser
  $\alpha^\star$, then
  \begin{equation}
    |\alpha^\star_{\mathcal{A}}-\alpha^\star|
      \leq h_{\mathcal{A}}.
  \end{equation}
  Separately, if $\phi\in C^2(\mathcal{O})$ for an open neighbourhood
  $\mathcal{O}$ of $[0,1]$, $\alpha^\star\in(0,1)$ is a maximiser, and
  $|\phi''(\alpha)|\leq B_\phi$ on $[0,1]$ with $B_\phi\geq0$, then
  \begin{equation}
    0\leq\phi(\alpha^\star)-\phi(\alpha^\star_{\mathcal{A}})
      \leq\frac{B_\phi}{8}h_{\mathcal{A}}^2.
  \end{equation}
\end{proposition}

\begin{proof}
  For the location bound, take grid points $a_{\mathrm L},a_{\mathrm R}$ bracketing
  $\alpha^\star$, and write $g=\alpha^\star_{\mathcal A}$.
  Uniqueness gives $\phi(x)<\phi(\alpha^\star)$ whenever
  $x\neq\alpha^\star$.  Suppose $g<a_{\mathrm L}$.
  If $a_{\mathrm L}=\alpha^\star$, this contradicts grid optimality.
  If $a_{\mathrm L}<\alpha^\star$, concavity and the strict value gap imply
  $\phi(g)<\phi(a_{\mathrm L})$, again contradicting grid optimality.
  The symmetric argument excludes $g>a_{\mathrm R}$.
  Thus $a_{\mathrm L}\leq g\leq a_{\mathrm R}$, and
  $|g-\alpha^\star|\leq a_{\mathrm R}-a_{\mathrm L}\leq h_{\mathcal A}$.
  When $\alpha^\star$ is a grid point, take
  $a_{\mathrm L}=a_{\mathrm R}=\alpha^\star$.

  For the value bound, interior constrained optimality is local
  optimality on the real line, so $\phi'(\alpha^\star)=0$.
  Choose a nearest grid point $b$ with
  $|b-\alpha^\star|\leq h_{\mathcal A}/2$ and reparametrise the
  segment between these points:
  \begin{equation}
    \psi(u)=\phi(\alpha^\star+u(b-\alpha^\star))-\phi(\alpha^\star),
    \qquad 0\leq u\leq1.
  \end{equation}
  Its derivatives are
  $\psi'(u)=\phi'(\alpha^\star+u(b-\alpha^\star))(b-\alpha^\star)$
  and
  $\psi''(u)=\phi''(\alpha^\star+u(b-\alpha^\star))(b-\alpha^\star)^2$.
  The segment lies in $[0,1]$.  Hence $\psi(0)=\psi'(0)=0$ and
  $\psi''(u)\geq-B_\phi(b-\alpha^\star)^2$.
  The two increment comparisons used above give the lower-bound formula
  also for zero initial derivative.  Applying it to $\psi$ at one yields
  $\phi(b)-\phi(\alpha^\star)\geq-B_\phi(b-\alpha^\star)^2/2$.
  Since $\phi(g)\geq\phi(b)$, the upper bound follows from the
  nearest-point distance.  Continuous optimality gives the lower bound.
\end{proof}

\paragraph{Sampled crossings.}
Suppose consecutive grid points $\alpha_+<\alpha_-$ satisfy
$\phi(\alpha_+)\geq0$ and $\phi(\alpha_-)<0$.
Continuity guarantees at least one zero in $[\alpha_+,\alpha_-)$.
Identifying that zero as the first positive zero requires additional
control of the path between samples.  Nonnegative values at every grid
point establish nonnegativity on the sampled set.

\begin{proof}
  The intermediate value theorem gives a zero in
  $[\alpha_+,\alpha_-]$.  The strict negative value at $\alpha_-$
  excludes that endpoint.
\end{proof}

\paragraph{Oracle endpoint comparison.}
The inclusion $\{0,1\}\subseteq\mathcal{B}$ implies
$U_{\mathrm{step}}\geq U_{\mathrm{halt}}$.

\begin{proof}
  Maximisation over $\mathcal{B}$ includes the utilities at both zero and
  one, so its value is at least their maximum.
\end{proof}

\section{Affine-Logit Utilities}
\label{sec:linear-readout}

This section specialises the path analysis to an average of log-softmax
utilities.  Fix an example $i$, a finite nonempty reference-token index set $M_i$, and a
finite nonempty vocabulary $\mathcal{V}$.  For each $j\in M_i$, let
$z_{0,j}$ and $v_j$ be real vectors indexed by $\mathcal V$, representing
the initial logits and their displacement, respectively.  Let
$z_j(\alpha)=z_{0,j}+\alpha v_j$ be the resulting affine logit path,
$p_j(\alpha):=\operatorname{softmax}(z_j(\alpha))$ its distribution, and
$y_{i,j}\in\mathcal{V}$ its reference label.  The component
$p_{j,a}(\alpha)$ is the probability of vocabulary entry $a$.
Suppose the utility along the segment satisfies
\begin{equation}
  \label{eq:affine-logit-utility}
  U(\gamma(\alpha))
    =\frac{1}{|M_i|}\sum_{j\in M_i}\log p_{j,y_{i,j}}(\alpha).
\end{equation}
All softmax probabilities are strictly positive.  For two such
distributions, use the finite KL divergence
$D_{\mathrm{KL}}(p\|\widetilde p)
 :=\sum_{a\in\mathcal{V}}p_a(\log p_a-\log\widetilde p_a)$.

\subsection{Directional Curvature}
\label{app:softmax-curvature}

Write the probability-weighted variance of a logit displacement as
\begin{equation}
  \label{eq:fisher-variance}
  V(p,v):=\sum_{a\in\mathcal V}p_a v_a^2
           -\left(\sum_{a\in\mathcal V}p_a v_a\right)^2.
\end{equation}
For the basis vector $e_{y_{i,j}}$ of the reference label, the
derivatives of each token utility are
\begin{equation}
  \label{eq:linear-readout-curvature}
  \begin{aligned}
    \frac{d}{d\alpha}\log p_{j,y_{i,j}}(\alpha)
      &=(e_{y_{i,j}}-p_j(\alpha))^\top v_j,\\
    \frac{d^2}{d\alpha^2}\log p_{j,y_{i,j}}(\alpha)
      &=-V(p_j(\alpha),v_j).
  \end{aligned}
\end{equation}
The variance is nonnegative and is strictly positive when the
displacement is nonconstant and all probabilities are positive.

\begin{proof}
  The softmax derivative satisfies
  $p'_{j,a}(\alpha)
    =p_{j,a}(\alpha)(v_{j,a}-p_j(\alpha)^\top v_j)$.
  Positivity allows differentiation of its logarithm by division by
  $p_{j,a}(\alpha)$.  At the reference label, this gives the first
  formula in Eq.~(\ref{eq:linear-readout-curvature}).
  Differentiate the resulting directional term:
  \begin{equation}
    \begin{aligned}
      \frac{d}{d\alpha}
        \left(v_{j,y_{i,j}}-\sum_a p_{j,a}(\alpha)v_{j,a}\right)
        &=-\sum_a p'_{j,a}(\alpha)v_{j,a}\\
        &=-\sum_a p_{j,a}(\alpha)v_{j,a}^2
           +\left(\sum_a p_{j,a}(\alpha)v_{j,a}\right)^2.
    \end{aligned}
  \end{equation}
  This is the negative variance in the second formula.
  For fixed $p,v$, put $\mu=\sum_a p_a v_a$.
  Expansion of the centred sum, using $\sum_a p_a=1$, gives
  $V(p,v)=\sum_a p_a(v_a-\mu)^2\geq0$.
  If the variance were zero, every nonnegative summand would vanish.
  With all $p_a>0$, this forces $v_a=\mu$ for every $a$.
  A nonconstant displacement therefore has strictly positive variance.
\end{proof}

\subsection{Endpoint Decomposition and Shape}
\label{app:endpoint-proof}

\begin{proof}[Proof of Proposition~\ref{prop:linear-endpoint}]
  For a logit vector $z$, write
  $\Lambda(z)=\log\sum_{a\in\mathcal V}\exp z_a$.
  The log-softmax identity is $\log p_a=z_a-\Lambda(z)$.
  For arbitrary endpoint logits $z,w$ and their distributions
  $p=\operatorname{softmax}(z)$ and
  $\widetilde p=\operatorname{softmax}(w)$, substitution into the
  finite KL sum and use of $\sum_a p_a=1$ give
  \begin{equation}
    \begin{aligned}
      D_{\mathrm{KL}}(p\|\widetilde p)
        &=\Lambda(w)-\Lambda(z)-\sum_a p_a(w_a-z_a),\\
      \log\widetilde p_a-\log p_a
        &=(e_a-p)^\top(w-z)-D_{\mathrm{KL}}(p\|\widetilde p).
    \end{aligned}
  \end{equation}
  To establish the sign, apply $\log x\leq x-1$ to
  $x=\widetilde p_a/p_a>0$, multiply by $p_a$, and rearrange.
  This gives $p_a-\widetilde p_a\leq
  p_a(\log p_a-\log\widetilde p_a)$.  Summing yields
  \begin{equation}
    0=\sum_a(p_a-\widetilde p_a)\leq
       D_{\mathrm{KL}}(p\|\widetilde p).
  \end{equation}
  Apply the endpoint identity to
  $z=z_{0,j}$, $w=z_{0,j}+v_j$, and reference label $y_{i,j}$,
  then average over $M_i$.  Finite-sum differentiation in
  Eq.~(\ref{eq:linear-readout-curvature}) identifies the average
  directional term with $A=\phi'(0)$.  The result is
  $\Delta U=A-\mathcal C$, with $\mathcal C\geq0$.
  The condition $A>0,\Delta U<0$ is consequently equivalent to
  $0<A<\mathcal C$.
  Finally, the same finite-sum derivative chain has second derivative
  $-|M_i|^{-1}\sum_j V(p_j(\alpha),v_j)$.
  Applying Eq.~(\ref{eq:pathwise-progress}) to this chain at one and
  comparing with the endpoint identity gives the integral expression
  for $\mathcal C$.
\end{proof}

\begin{corollary}[Shape of a Finite-Step Failure]
  \label{cor:linear-failure-shape}
  Under Eq.~(\ref{eq:affine-logit-utility}), suppose at least one
  $v_j$ is nonconstant across the vocabulary.  Then $\phi$ is strictly
  concave on $[0,1]$.  If additionally $A>0$ and $\phi(1)<0$, it has
  a unique maximiser on $[0,1]$, lying in $(0,1)$, and a unique
  positive zero in $(0,1)$.
\end{corollary}

\begin{proof}
  Equation~(\ref{eq:linear-readout-curvature}) expresses $\phi''$ as
  the negative average of nonnegative variances.  The nonconstant
  displacement makes at least one variance strictly positive at every
  scale.  Thus $\phi''<0$ and $\phi$ is strictly concave.
  Since $A>0$, the local positivity argument gives
  $\phi(\alpha)>0$ for $0<\alpha<\varepsilon$.
  Choose $\alpha_0=\min\{\varepsilon/2,1/2\}$.
  The smallest maximiser exists by Appendix~\ref{app:quadratic-max},
  and its value is at least $\phi(\alpha_0)>0$.
  It is therefore different from both endpoints, whose values are zero
  and negative.  Strict concavity gives uniqueness of the maximiser.

  Continuity on $[\alpha_0,1]$ gives a zero $r$ in that interval.
  Since $\alpha_0>0$ and $\phi(1)<0$, it lies in $(0,1)$.
  If $0<r<s\leq1$ were two zeros, then
  $r=(1-r/s)0+(r/s)s$ lies strictly between zero and $s$.
  Strict concavity would give $\phi(r)>0$, contradicting its being a
  zero.  Interchanging $r,s$ excludes the reverse ordering, proving
  uniqueness.
\end{proof}

\section{Looped Language Model Instantiation}
\label{app:instantiation}

\subsection{States and Readout Boundaries}
\label{app:state-boundaries}

For a model with frozen parameters $\theta$ and example $i$, fix the input
representation $E_i$, tokenisation, initialisation, and execution policy $\pi$, including batch
composition and padding.  An observed recurrent transition is
\begin{equation}
  H_{i,t+1}=\widehat F_{\theta,\pi}(H_{i,t};E_i),
  \qquad D_{i,t}=H_{i,t+1}-H_{i,t}.
\end{equation}
Here $\widehat F_{\theta,\pi}$ is the implemented recurrent transition
under execution policy $\pi$, and $t$ counts recurrent iterations.
Fixing $i,t,\theta,\pi,E_i$ specifies the update $H^+=F(H)$ in
\S~\ref{sec:task-utility}.  Let $n_i$ be the sequence length and $d$ the
hidden dimension.  The single-sequence state space is
$\mathbb{R}^{n_i\times d}$ with its Frobenius inner product.  A joint
option state belongs to the product of the option-specific sequence
spaces, with inner product equal to the sum of the component Frobenius
inner products.

For Ouro, the state boundary lies after the shared stack and its output
normalisation, immediately before the linear language-model head.
For Huginn it lies after the recurrent core and before the fixed nonlinear
coda and output head.  The recurrent endpoints are produced in BF16.
The fixed readout is evaluated in FP32, and the final log-softmax over joint
option scores is evaluated in float64.  Derivatives hold the endpoints
and their displacement fixed and pass only through the readout.
The mathematical statements apply to the real-valued readout continuation
under their stated smoothness conditions.  The numerical calculations
evaluate this analysis at the computed recurrent states.

\subsection{Task Utilities}
\label{app:utilities}

\paragraph{Reference solutions.}
Let $\mathcal V$ be the model's finite nonempty vocabulary.  For a
mathematical problem $x_i$ and its nonempty reference solution $y_i$,
let $M_i\subseteq\{1,\ldots,|y_i|\}$ be the nonempty set of evaluated
solution-token indices.  Define
\begin{equation}
  \label{eq:reference-utility}
  U_i(H):=\frac{1}{|M_i|}\sum_{j\in M_i}
    \log p_\theta(y_{i,j}\mid x_i,y_{i,<j};H).
\end{equation}
Here $p_\theta$ is the model's conditional token distribution,
$y_{i,j}$ is reference token $j$, and $y_{i,<j}$ is its preceding
reference prefix.  The length $|y_i|$ counts reference tokens.
Each token is read at its preceding prediction position, and the complete
reference solution is retained.  With $i$ fixed, this is the utility $U$
used in the path analysis.  Ouro's linear head supplies the affine logits
in Appendix~\ref{sec:linear-readout}.

\begingroup
\interlinepenalty=10000
\postdisplaypenalty=10000
\paragraph{Joint options.}
For a problem $x$ with $K\geq1$ nonempty option sequences $y^{(k)}$,
write their teacher-forced state tuple as
$\mathbf H=(H^{(1)},\ldots,H^{(K)})$ and let
$k^\star\in\{1,\ldots,K\}$ be the correct option.  Suppressing the
example index, define
\begin{equation}
  \label{eq:option-utility}
  \begin{aligned}
    s_k(H^{(k)})
      &:=\frac{1}{|y^{(k)}|}\sum_{j=1}^{|y^{(k)}|}
        \log p_\theta(y^{(k)}_j\mid x,y^{(k)}_{<j};H^{(k)}),\\
    U(\mathbf H)
      &:=s_{k^\star}(H^{(k^\star)})
         -\log\sum_{k=1}^K\exp s_k(H^{(k)}).
  \end{aligned}
\end{equation}
\endgroup
The score $s_k$ is the mean log probability of option $k$, and
$y^{(k)}_{<j}$ is that option's prefix preceding token $j$.
These scores are evaluated under teacher forcing: the reference prefix is
provided when the next token is scored.  HellaSwag uses $K=4$.
Differentiation and intervention act on the whole
tuple, including the cross-option dependence introduced by log-softmax.
The affine-logit special case concerns the reference-solution utility.
Huginn's nonlinear readout and the joint option utility use the general
pathwise derivatives.

Both utilities measure support for the reference under teacher forcing.
Generation accuracy and utility are separate outcomes.  Comparisons hold
the checkpoint, conditioning, state boundary, and utility fixed, and report
magnitudes within each task setting.

\subsection{Scales and Comparisons}
\label{app:measurement-scales}

The step-scale prediction is evaluated on
$\mathcal A=\{0,0.05,\ldots,1\}$, using the original $A,Q$ without
refitting them to the measured curve.  The primary subset has $A>0,Q<0$.
The error to the grid optimum is distinct from the continuous error in
Theorem~\ref{thm:integrated-step}.  Appendix~\ref{app:grid} states the
conditions under which grid resolution controls that difference.
The measured prediction accuracy does not estimate $C(1)$ or the
third-derivative bound $M$.  Those bounds are evaluated separately
(Appendix~\ref{app:cloud-bounds}).

The intervention set is $\mathcal B=\{0,0.25,0.5,0.75,1\}$, with
prespecified step contraction $\alpha_c=0.25$.  Every option of a joint state
receives the same scale.  Recovery frequency and mean gain are evaluated
on the finite-step failures selected by Definition~\ref{def:finite-failure}.
Oracle comparisons use all analysed examples and the same fixed action
set $\mathcal B$.  They use reference utility after the original
displacement has been computed, with the effect measured before further
recurrence.

\section{Additional Experimental Details}
\label{app:exp-details}

\subsection{Coverage and Depth Selection}
\label{app:exp-coverage}

The complete utility sweep evaluates the seven labelled splits in
Table~\ref{tab:evaluation-coverage} for each of the three models.
There are 19362 questions per model, giving 58086 model--question
observations.  Questions are shared across models.  Both Ouro models
are evaluated at depths one through eight, then 12, 16, 24, and 32.
Huginn is evaluated at depths 4, 8, 16, 19, 20, 24, 32, 33, 34,
48, 64, 96, and 128.
Every specified depth is retained, including depths after a measured
loss.  At each depth, the running-best comparison uses the earlier
evaluated depth with the largest mean utility.  The reported interval
is the paired interval for the resulting per-question difference.
Depths and datasets have been inspected during the study, and these
pointwise intervals have no adjustment for selection or multiple
comparisons.

\begin{table}[t]
\centering
\caption{\textbf{Depth-sweep coverage and observed utility declines.} Peak is the evaluated depth with highest mean reference utility. Decline is the earliest running-best decrease with a pointwise 95\% paired interval below zero; a dash indicates none on the evaluated grid. Non-adjacent depth pairs represent cumulative changes. Val. denotes validation.}
\label{tab:evaluation-coverage}
\small
\setlength{\tabcolsep}{2pt}
\begin{tabular*}{\linewidth}{@{\extracolsep{\fill}}llrcccccc@{}}
\toprule
\multirow{2}{*}{Task} & \multirow{2}{*}{Split} & \multirow{2}{*}{$N$} & \multicolumn{2}{c}{Ouro-1.4B} & \multicolumn{2}{c}{Ouro-2.6B} & \multicolumn{2}{c}{Huginn} \\
\cmidrule(lr){4-5}\cmidrule(lr){6-7}\cmidrule(l){8-9}
 & & & Peak & Decline & Peak & Decline & Peak & Decline \\
\midrule
MATH-500 & Test & 500 & 4 & $4\to5$ & 4 & $4\to5$ & 24 & $24\to48$ \\
GSM8K & Test & 1319 & 4 & $4\to5$ & 4 & $4\to5$ & 16 & $16\to19$ \\
HellaSwag & Val. & 10042 & 2 & $2\to3$ & 6 & $6\to24$ & 32 & $32\to34$ \\
ARC-C & Test & 1172 & 4 & $4\to5$ & 4 & $4\to5$ & 19 & --- \\
PIQA & Val. & 1838 & 4 & $4\to12$ & 7 & $7\to24$ & 64 & --- \\
BoolQ & Val. & 3270 & 1 & $1\to3$ & 4 & $4\to5$ & 24 & $24\to32$ \\
CommonsenseQA & Val. & 1221 & 4 & $4\to5$ & 1 & $1\to2$ & 20 & $20\to24$ \\
\bottomrule
\end{tabular*}
\vspace{-4mm}
\end{table}

The magnitude and interpretation of a decline depend on the task.
Huginn's first confirmed running-best loss on HellaSwag is a cumulative
$32\to34$ change of approximately $-8.45\times10^{-5}$.
For Ouro-1.4B on BoolQ, the model predicts Yes on 3268 of 3270 questions
at depths three and four, and on all questions at depth twelve.
This response imbalance accompanies the utility curve and limits its
interpretation as a deterioration of reasoning.
The original Ouro-1.4B/HellaSwag validation set supplies a separate
comparison between utility and discrete correctness.  Across its 1000
questions, the $4\to5$ mean utility change is $-0.008006$, with interval
$[-0.010838,-0.004991]$, while option accuracy rises from 68.4\% to
69.2\%.  The smooth reference utility and the selected answer can
therefore move differently under the same recurrent update.
The main mechanism comparison instead holds the transition fixed across
tasks within each model and includes all nine conditions, with 35583
model--question observations in total.  It reuses the states and
execution settings of the complete utility sweep.

\subsection{Evaluation Sets and Numerical Protocol}
\label{app:exp-protocol}

\paragraph{Transition and sample provenance.}
The original Ouro/MATH-500 study used an independent 128-question
training pilot to identify the depth range around the $4\to5$
transition, then evaluated the full 500-question test set.
The original Huginn/GSM8K evaluation set consists of 256 training questions:
all 256 source identifiers match the training split and none match the
test split.  The $19\to20$ transition was the earliest consecutive
decrease with a negative pointwise mean-gain interval within the
examined $16$--$24$ range on those same questions.
The original Ouro-1.4B/HellaSwag study fixed $4\to5$ before measurement
and evaluated 1000 validation questions.
The CommonsenseQA extension evaluates 1024 training questions per
model.  Ouro retains $4\to5$, and Huginn uses $33\to34$, selected by the
earliest negative consecutive-change interval in a dense sweep of the
same evaluation set.  Earlier nested subsets of these questions do not provide
independent replications.

The complete mechanism matrix applies $4\to5$ to both Ouro models
and $19\to20$ to Huginn on each full split.  A separate Huginn/GSM8K
test evaluation also held $19\to20$ fixed and found 107 finite-step
failures among 1319 questions.  Its logged execution and initialisation
settings differ from those of the complete matrix, which gives 102
such failures.  Each result belongs to its own observed transitions, and
the matrix uses only the latter run.  The split distinction concerns
evaluation provenance and does not establish absence from model
pretraining.

\paragraph{Execution and differentiation.}
The complete matrix reconstructs each original batch with its recorded
tokenisation, masks, candidate grouping, and natural padding width.
For Huginn, the complete utility sweep initialises each question with a
CPU BF16 truncated-normal state using seed $20260904$ plus its source
index, and candidate sequences share the question's seed and initialisation
rule.  The original 256-question GSM8K study instead continues cached
depth-eight states with batch size 32.  Those settings are preserved
when analysing its step contractions and scale predictions.

All nine mechanism conditions compute $A$ and $Q$ by automatic
differentiation and a Hessian--vector product through the actual
readout utility, without constructing a full Hessian.  TF32 is disabled.
For HellaSwag, the product includes the cross-candidate terms of the
joint option utility.  Replayed endpoint utilities differ from the
source values by at most $8.89\times10^{-16}$.
On the first 16 questions of each condition, an additional scalar
second-derivative calculation along the normalised full valid-token
chord agrees with the Hessian--vector calculation to absolute error
at most $10^{-6}$ or relative error at most $10^{-4}$, and all 144 checks
pass this disjunctive criterion.  These checks compare two derivative
calculations through the same numerical readout.

\paragraph{Intervals and effect margins.}
For the complete matrix, bootstrap resampling uses whole question
records with seed 20260904.  Each replicate recomputes the number of
finite-step failures divided by the number of harmful updates, so the
denominator varies with the resample.
Table~\ref{tab:mechanism-intervals} reports these intervals together
with the count satisfying both $A>10^{-4}$ and
$-\Delta U>10^{-4}$.  All nine conditions have zero neutral updates.
Contraction intervals resample the selected finite-step failures, and
primary scale-comparison intervals resample the $A>0,Q<0$ subset.
These conditional intervals describe the selected population and do
not adjust for the earlier choice of transition.

\begin{table}[t]
\centering
\caption{\textbf{Finite-step failure shares and margin counts.}
Shares are relative to harmful updates. Margin counts require both
$A>10^{-4}$ and $-\Delta U>10^{-4}$.
Settings follow Table~\ref{tab:mechanism-matrix}.}
\label{tab:mechanism-intervals}
\small
\setlength{\tabcolsep}{5pt}
\begin{tabular*}{0.8\linewidth}{@{\extracolsep{\fill}}lccc@{}}
\toprule
Model & Task & \makecell{Failure share (\%)\\with 95\% CI}
& \makecell{Failures beyond\\the margin} \\
\midrule
\multirow{3}{*}{Ouro-1.4B}
& MATH-500 & $49.18\;[44.29,54.10]$ & 177 \\
& GSM8K & $40.24\;[37.35,43.14]$ & 421 \\
& HellaSwag & $4.52\;[3.97,5.04]$ & 253 \\
\addlinespace[3pt]
\multirow{3}{*}{Ouro-2.6B}
& MATH-500 & $28.88\;[24.10,33.64]$ & 90 \\
& GSM8K & $37.19\;[33.79,40.32]$ & 321 \\
& HellaSwag & $1.86\;[1.49,2.27]$ & 78 \\
\addlinespace[3pt]
\multirow{3}{*}{Huginn}
& MATH-500 & $22.17\;[16.35,27.89]$ & 32 \\
& GSM8K & $14.53\;[12.02,17.09]$ & 84 \\
& HellaSwag & $1.38\;[1.05,1.73]$ & 24 \\
\bottomrule
\end{tabular*}
\vspace{-4mm}
\end{table}

Figure~\ref{fig:contraction} gives the quarter-step estimates
and intervals for the original evaluation sets.

\paragraph{Task extensions and effect margins.}
Recovery at a shared scale varies across the examined paths.  Nine of
Huginn/GSM8K's 13 original recoveries gain at most $10^{-4}$, while 19
of 20 HellaSwag recoveries exceed that margin.  On the CommonsenseQA
training subsets, the quarter step recovers 13 of 22 finite-step failures
for Ouro-1.4B, 18 of 28 for Ouro-2.6B, and three of four for Huginn.
These observations extend the recovery phenomenon to the additional
model--task settings.  Huginn's $33\to34$ transition was selected on
its training subset; evaluation on the complete validation split leaves
the sign of the mean gain unresolved.
Figure~\ref{fig:contraction} displays the original quarter-step estimates
and their uncertainty, including all three CommonsenseQA evaluation sets.

\begin{figure}[t]
\centering
\includegraphics[width=\linewidth]{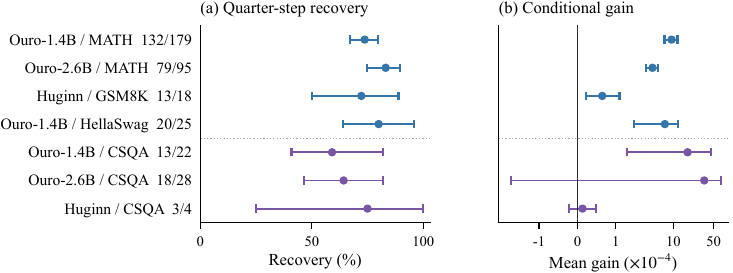}
\caption{\textbf{Quarter-step recovery.}
Each update uses one quarter of its proposed displacement along the same
direction. (a) Fraction of finite-step failures with positive gains in
reference utility, with recovered/total counts. (b) Mean gain over
the same failure set, including unrecovered cases, on a symmetric-log
scale. Purple rows denote CommonsenseQA.}
\label{fig:contraction}
\end{figure}

\paragraph{Exact KL correction for the affine readout.}
\label{app:exp-kl}
For the original Ouro/MATH-500 evaluation sets, we evaluate the mean KL term
$\mathcal C$ in Eq.~(\ref{eq:linear-endpoint}) using the endpoint
logits and a log-sum-exp calculation.  This computes the correction
independently of the measured difference $A-\Delta U$.
Table~\ref{tab:exact-kl} reports the maximum absolute residual of
$\Delta U-(A-\mathcal C)$ and the number of finite-step failures
satisfying the exact decomposition's inequality numerically.
The single exception on Ouro-2.6B has
$\Delta U=-3.576\times10^{-7}$ and signed residual
$-6.044\times10^{-7}$.  Its full-step classification is retained,
but the strict inequality is unresolved at this numerical precision.
This calculation applies to the affine reference-token readout.
The nine-condition comparison uses automatic differentiation of each
actual utility, including the nonlinear coda and joint option readout.

\begin{table}[t]
\centering
\caption{\textbf{Numerical check of the exact endpoint correction.} Both rows use the original MATH-500 test evaluation set ($N=500$, $4\to5$). The residual is $|\Delta U-(A-\mathcal C)|$, with $\mathcal C$ computed independently from endpoint logits. The last column counts finite-step failures satisfying $0<A<\mathcal C$ numerically.}
\label{tab:exact-kl}
\small
\begin{tabular}{@{}lrr@{}}
\toprule
Model & Maximum residual & Inequality satisfied \\
\midrule
Ouro-1.4B & $2.661\times10^{-6}$ & 179/179 \\
Ouro-2.6B & $3.348\times10^{-6}$ & 94/95 \\
\bottomrule
\end{tabular}
\vspace{-4mm}
\end{table}

\subsection{State and Output Diagnostics}
\label{app:exp-ranking}

The diagnostic comparison uses all valid token positions in the
original four evaluation sets.  State norms use the Frobenius norm, with all
candidate states concatenated for multiple-choice questions.
The signals are $\|D\|$, relative change
$\|D\|/(\|H\|+10^{-12})$, and cosine change
$1-\langle H,H^+\rangle/(\|H\|\|H^+\|)$.  The norms in the cosine
denominator are positive in these records.
Output KL divergence uses the current token distribution as its first
argument and the next-state distribution as its second, over the full
vocabulary at prediction positions.  Entropy change is next-state
entropy minus current-state entropy at those positions.
For HellaSwag, output diagnostics first average over prediction tokens
within each candidate, then equally over the four candidates.

Table~\ref{tab:diagnostic-ranking} reports the raw signal's Spearman
correlation with $\Delta U$ and the harmful-update AUROC after a fixed
risk orientation.  Larger risk corresponds to the negative of each
signal except entropy change, whose orientation is positive.
Orientations are retained when AUROC is below 0.5, and no affine
mapping is fitted to the evaluation outcomes.  The derivative scores
use the reference utility, and the state and output diagnostics omit that
task-specific derivative information while sharing the same
teacher-forced states.

\begin{table}[t]
\centering
\caption{\textbf{State, output, and task-geometry diagnostics.} Columns use the original MATH-500 test evaluation sets ($N=500$ each), GSM8K training subset ($N=256$), and HellaSwag validation set ($N=1000$). Raw correlations retain signal orientation, and AUROC uses the fixed risk signs of Appendix~\ref{app:exp-ranking}.}
\label{tab:diagnostic-ranking}
\small
\setlength{\tabcolsep}{5pt}
\begin{tabular}{@{}lrrrr@{}}
\toprule
Signal & \makecell{Ouro-1.4B\\MATH-500} & \makecell{Ouro-2.6B\\MATH-500} & \makecell{Huginn\\GSM8K} & \makecell{Ouro-1.4B\\HellaSwag} \\
\midrule
\multicolumn{5}{l}{\textit{Raw Spearman correlation with $\Delta U$}} \\
Update norm & -0.2843 & 0.0471 & 0.1006 & -0.0090 \\
Relative change & -0.2725 & -0.0620 & 0.0010 & -0.0457 \\
Cosine change & -0.2724 & -0.0640 & 0.0010 & -0.0457 \\
Output KL & -0.2357 & -0.0985 & -0.1490 & -0.0057 \\
Entropy change & -0.1573 & -0.0344 & -0.0110 & 0.0350 \\
$A$ & 0.8276 & 0.9640 & 0.9951 & 0.9622 \\
$A+Q$ & \textbf{0.9929} & \textbf{0.9973} & \textbf{0.9999} & \textbf{0.9982} \\
\midrule
\multicolumn{5}{l}{\textit{Harmful-update AUROC}} \\
Update norm & 0.2893 & 0.4852 & 0.5122 & 0.4983 \\
Relative change & 0.3438 & 0.5241 & 0.4701 & 0.4960 \\
Cosine change & 0.3440 & 0.5229 & 0.4701 & 0.4961 \\
Output KL & 0.4390 & 0.5290 & 0.4279 & 0.5119 \\
Entropy change & 0.5826 & 0.5357 & 0.5105 & 0.4848 \\
$A$ & 0.9411 & 0.9764 & 0.9974 & 0.9811 \\
$A+Q$ & \textbf{0.9974} & \textbf{0.9990} & \textbf{0.9998} & \textbf{0.9988} \\
\bottomrule
\end{tabular}
\vspace{-4mm}
\end{table}

\subsection{Scale Prediction and Curvature Measurements}
\label{app:exp-scale-details}

\begin{table}[t]
\centering
\caption{\textbf{Local quadratic scale prediction.} Primary subset $A>0,Q<0$. Correlation and scale MAE compare the continuous prediction with the 21-point grid optimum, and utility regret uses the nearest grid action. CommonsenseQA rows use exploratory training subsets.}
\label{tab:step-prediction}
\small
\setlength{\tabcolsep}{3pt}
\begin{tabular}{@{}lrrrrr@{}}
\toprule
Model / task &  $n$ & Spearman & Scale MAE & \makecell{First-order\\regret} & \makecell{Quadratic\\regret} \\
\midrule
\makecell[l]{Ouro-1.4B / MATH-500\\test, $N=500$, $4\to5$} & 315 & 0.9951 & 0.0211 & $6.714\times10^{-3}$ & $\boldsymbol{1.747\times10^{-5}}$ \\
\makecell[l]{Huginn / GSM8K\\train, $N=256$, $19\to20$} & 129 & 0.9998 & 0.0072 & $9.681\times10^{-5}$ & $\boldsymbol{3.482\times10^{-7}}$ \\
\makecell[l]{Ouro-1.4B / CommonsenseQA\\train, $N=1024$, $4\to5$} & 302 & 0.8726 & 0.0146 & $2.044\times10^{-3}$ & $\boldsymbol{1.919\times10^{-4}}$ \\
\makecell[l]{Ouro-2.6B / CommonsenseQA\\train, $N=1024$, $4\to5$} & 380 & 0.8369 & 0.0155 & $5.231\times10^{-3}$ & $\boldsymbol{7.314\times10^{-4}}$ \\
\makecell[l]{Huginn / CommonsenseQA\\train, $N=1024$, $33\to34$} & 275 & 0.9275 & 0.0011 & $5.742\times10^{-6}$ & $\boldsymbol{1.292\times10^{-7}}$ \\
\bottomrule
\end{tabular}
\vspace{-4mm}
\end{table}

The CommonsenseQA extensions have scale correlations of 0.873 and 0.837
for the two Ouro models and 0.928 for Huginn, with 268 of 275 Huginn
grid optima at one.  These estimates use exploratory training subsets,
as specified in Appendix~\ref{app:exp-protocol}.

The quadratic rule maximises $q$ on $[0,1]$ as specified in
Eq.~(\ref{eq:predicted-step}).  It compares the endpoints and, when
applicable, the interior concave maximum.  Ties choose the smaller
scale.  Continuous scale errors are computed before mapping to the
21-point grid, and utility comparisons use the nearest grid point, with
equal distances resolved toward the smaller point.
Fixed-scale baselines use $0,0.25,0.5,0.75,1$.
Figure~\ref{fig:scale-prediction}(e) compares all five choices
with the quadratic rule on the original primary
evaluation sets.  The paired quadratic-minus-first-order utility gain is
$0.006697$ with interval $[0.005622,0.007794]$ for Ouro/MATH-500,
and $9.646\times10^{-5}$ with interval
$[5.763\times10^{-5},1.427\times10^{-4}]$ for Huginn/GSM8K.
Selecting the best fixed action on these same evaluation subsets gives
a descriptive hindsight comparator.

The all-question scale correlations and MAEs are 0.9932 and 0.0133 on
Ouro/MATH-500, and 0.9983 and 0.0036 on Huginn/GSM8K.
On the original Huginn primary subset, 93 of 129 questions have both
predicted and empirical grid optima at one.  Restricting descriptively
to its 35 interior grid optima gives correlation 0.9924 and MAE 0.0262.
Twelve Ouro questions have $A>0$ but a grid optimum at zero: a grid
with spacing 0.05 need not resolve the beneficial interval near the
origin.

\paragraph{Predicting the progress boundary.}
\label{app:exp-boundary}
Within each primary subset, we locate the first grid point with negative
gain.  Its immediately preceding grid point and this first negative
point define the observed crossing interval.  The numerical hit rule
includes the lower endpoint and excludes the upper endpoint.
Of 179 Ouro and 18 Huginn questions with an observed crossing, the
predicted root $r^{(2)}$ lies in this interval for 121 and 13,
respectively (Table~\ref{tab:progress-boundary}).
The remaining 136 Ouro and 111 Huginn primary questions have no
observed negative gain through scale one, and their roots receive no
numerical error by assigning the endpoint one as a target.

Continuity gives a zero in an observed crossing interval.  In the
affine reference-token case, concavity connects this crossing to the
end of the nonnegative prefix.  For a general nonlinear readout,
negative excursions between earlier grid points are not excluded, so
the interval need not locate the first continuous boundary $r_1$.
Within $A>0,Q<0$, the condition $r^{(2)}<1$ is algebraically
equivalent to $A+Q<0$.  Its classification accuracy therefore reuses
the endpoint prediction, while interval hits assess the additional
prediction of a location along the segment.

\begin{table}[t]
\centering
\caption{\textbf{Predicting the observed end of progress.} Within $A>0,Q<0$, the quadratic root against the first observed nonnegative-to-negative crossing interval on the 21-point grid. Hits require the root to lie in that width-0.05 interval. The last column counts primary questions with no observed negative gain through scale one.}
\label{tab:progress-boundary}
\small
\setlength{\tabcolsep}{5pt}
\begin{tabular}{@{}lrrrrr@{}}
\toprule
Model / task & Primary $n$ & Crossings & Hits & Hit rate (\%) & \makecell{No negative\\gain observed} \\
\midrule
\makecell[l]{Ouro-1.4B / MATH-500\\test, $N=500$, $4\to5$} & 315 & 179 & 121 & 67.6 & 136 \\
\makecell[l]{Huginn / GSM8K\\train, $N=256$, $19\to20$} & 129 & 18 & 13 & 72.2 & 111 \\
\bottomrule
\end{tabular}
\vspace{-4mm}
\end{table}

\paragraph{Directional curvature along the path.}
The five-point derivative measurements retain the original unnormalised
$D$ at every scale.  The interquartile ranges of
$\phi''(0.5)/\phi''(0)$ are $[1.003,1.133]$ for Ouro and
$[0.954,1.039]$ for Huginn on the primary subsets.
The median maximum adjacent curvature difference, divided by
$|\phi''(0)|+\varepsilon$, is 0.0951 and 0.0294, respectively.
Here $\varepsilon$ is $10^{-6}$ times the median of the original
$|2Q|$ over the corresponding complete evaluation set, fixed before restricting
to the primary subset.  The adjacent difference is not divided by
the scale spacing.  Curvature ratios omit origins with
$|\phi''(0)|\leq\varepsilon$.
Negative curvature is observed at every sampled point up to the
empirical optimum in all 315 Ouro and 129 Huginn primary cases.
Of these cases, 84 and seven, respectively, include only the current state
in that check.  This sampled-prefix observation does not establish
concavity throughout the segment.

\subsection{FP64 Path Analysis and Curvature-Based Scale Selection}
\label{app:exp-fp64}

\paragraph{Fixed states and numerical readouts.}
The refined analysis uses all 500 original Ouro-1.4B/MATH-500 questions
and all 256 original Huginn/GSM8K training questions at their respective
$4\to5$ and $19\to20$ transitions.  The recurrent computation and
chosen endpoints are retained.  Readout parameters and states are
evaluated in FP64, with $A,Q$ recomputed for that numerical path.
For Huginn, the original recurrent batch size remains 32.  The analysis
readout processes two questions at a time at the original padding width,
using explicit causal attention to support third-order differentiation.
The maximum FP32 replay differences induced by the smaller readout
batch are $1.789\times10^{-6}$ at the current state and
$1.550\times10^{-6}$ at the endpoint, and the explicit-coda bridge differs
by at most $1.014\times10^{-6}$.
Primary membership and harmful-update labels are unchanged in both
models.  The maximum absolute FP64--FP32 differences in $A,Q,\Delta U$
are, respectively, $2.331\times10^{-6}$, $7.098\times10^{-7}$,
and $2.724\times10^{-6}$ for Ouro, and $2.327\times10^{-7}$,
$4.918\times10^{-8}$, and $5.403\times10^{-6}$ for Huginn.
The maximum state-entry gaps between $H+D$ and the stored next state
are $5.961\times10^{-8}$ and $4.657\times10^{-10}$, respectively.

\paragraph{Refined numerical optima.}
Ouro uses concavity to select an endpoint or bracket a zero of $\phi'$.
The 269 interior brackets have width at most $10^{-4}$, with positive
left derivative and nonpositive right derivative numerically.
Propagating these brackets gives the scale-error and regret intervals
in Table~\ref{tab:fp64-path-prediction}.  The remaining optima are
185 left endpoints and 46 right endpoints.
Huginn uses 257 adaptive derivative locations per question, with
additional local root refinement and evaluation of $\widehat\alpha$.
Its numerical optimum estimate is the smallest maximiser among the evaluated candidates:
126 estimates lie at zero, 94 at one, and 36 in the interior.
The primary subset contains 93 right endpoints and all 36 interior
estimates.  Including $\widehat\alpha$ among the candidates ensures
nonnegative numerical regret by construction.  That regret is the
best evaluated utility minus $\phi(\widehat\alpha)$, and a missed maximum
could increase the true continuous regret.
Equal numerical maxima select the smaller scale, so flat paths can
make the reported location sensitive to finite-precision ties.
Both models evaluate the continuous prediction without grid snapping.
Ouro's intervals propagate numerical root brackets, whereas Huginn's
entries compare with the searched candidate.  Neither incorporates
outward-rounded arithmetic or provides a floating-point certificate.

\begin{table}[t]
\centering
\caption{\textbf{Quadratic prediction against refined FP64 optimum estimates.} Original evaluation sets of 500 Ouro-1.4B/MATH-500 test and 256 Huginn/GSM8K training questions, with primary-subset membership retained from the original analysis. Ouro entries are numerical intervals propagated from derivative-root brackets, and Huginn entries use the best evaluated candidate. These intervals are numerical, not confidence intervals. Interior and failure rows are descriptive subsets.}
\label{tab:fp64-path-prediction}
\small
\setlength{\tabcolsep}{5pt}
\begin{tabular}{@{}llrrr@{}}
\toprule
Model & Population & $n$ & Mean scale error & Mean numerical regret \\
\midrule
Ouro-1.4B & All & 500 & $[0.010221,0.010254]$ & $[1.150383,1.150385]\times10^{-5}$ \\
Ouro-1.4B & Primary & 315 & $[0.016224,0.016276]$ & $[1.826005,1.826008]\times10^{-5}$ \\
Ouro-1.4B & Interior & 269 & $[0.018928,0.018989]$ & $[2.134553,2.134556]\times10^{-5}$ \\
\addlinespace[3pt]
Huginn & All & 256 & 0.002447 & $1.0837\times10^{-7}$ \\
Huginn & Primary & 129 & 0.004855 & $2.1506\times10^{-7}$ \\
Huginn & Interior & 36 & 0.017399 & $7.7063\times10^{-7}$ \\
Huginn & Finite-step failures & 18 & 0.006354 & $1.5575\times10^{-7}$ \\
\bottomrule
\end{tabular}
\vspace{-4mm}
\end{table}

\paragraph{Figure construction.}
\label{app:exp-figures}
Figure~\ref{fig:direction-progress} displays every record in the
matched mechanism matrix, with identical symmetric-log transforms on
the two axes.  Figure~\ref{fig:scale-prediction}(a--d) uses the separately
recomputed FP64 paths.  For each mathematical evaluation set, finite-step
failures are ordered by their numerical optimum estimate, with record ID
breaking ties, and the lower-median record supplies the example curve
in panels (a,b).  This selection does not use the quadratic
approximation error.  Lines connect the stored path evaluations.  The
quadratic is evaluated from the unchanged FP64 coefficients.  Every
primary prediction enters panels (c,d).  Identical coordinates are
combined with marker area increasing as the square root of their count.
No jitter is applied.  Ouro's numerical root intervals are drawn
vertically, and their midpoints provide display coordinates.
Panel (e) uses the original 21-point-grid regrets and conditional 95\%
bootstrap intervals on the original primary subsets.
All displayed intervals are strictly positive, allowing a logarithmic
axis; the zero-regret grid oracle is omitted.  The quadratic action is
displayed separately from the fixed-scale sequence.  Its plotted
regret is evaluated at the nearest grid action, matching
Table~\ref{tab:scale-prediction-main}.
Figure~\ref{fig:contraction} retains the original reported bootstrap
intervals.  The paired FP64 quarter-step and bound-selected gains are
given in Table~\ref{tab:curvature-selected-scale}.  The Python
generator records the selected IDs and hashes of all input files.
\begin{figure}[t]
\centering
\includegraphics[width=\linewidth]{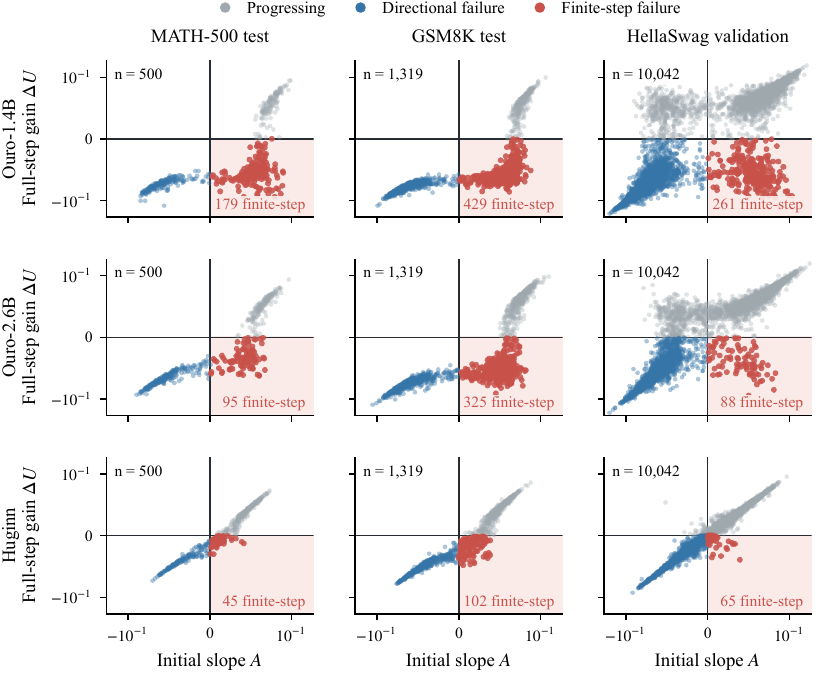}
\caption{\textbf{Local improvement and full-step harm.}
Each point compares the initial slope $A$ with the full-step gain
$\Delta U$ for a model--question pair in Table~\ref{tab:mechanism-matrix}.
Red points mark updates whose directions locally increase reference
utility but whose full steps decrease it. Blue points mark harmful
updates with nonpositive initial slopes, and grey points mark full
updates that increase reference utility.}
\label{fig:direction-progress}
\end{figure}

\paragraph{Derivative envelopes and refinement.}
Ouro evaluates directional derivatives through softmax moments.
The numerical envelope construction combines global moment bounds
with a fourth-derivative bound on each of 256 contiguous intervals
covering $[0,1]$.  It returns upper bounds for $|\phi'''|$ and
$-\phi''$ by propagating endpoint derivative values across each
interval and taking the maximum of the interval bounds.
The resulting constants are evaluated using ordinary FP64 arithmetic.
On the primary subset, the median ratio of the computed $L_D$ to the
largest sampled $-\phi''$ is 1.0146.  This ratio compares the envelope
with sampled extrema, not with the exact continuous supremum.
The same cellwise construction supplies the constants $L_\ell$ of
Proposition~\ref{prop:curvature-partition}.  The resulting integrated
and signed-residual scale bounds, and the regret bounds of
Theorem~\ref{thm:integrated-step}, are evaluated on the complete
mathematical benchmarks in Appendix~\ref{app:cloud-bounds}.

Huginn computes third derivatives by automatic differentiation.
The median relative discrepancies with central differences of $\phi''$
are $7.137\times10^{-6}$, $1.784\times10^{-6}$, and
$4.460\times10^{-7}$ at spacings 0.01, 0.005, and 0.0025.
At the smallest spacing, the maximum absolute discrepancy is
$1.069\times10^{-8}$.  Increasing the adaptive point count from 129
to 257 changes the best sampled utility by at most
$3.718\times10^{-10}$ and its location by at most 0.000489.
The sampled maxima of $|\phi'''|$ and $[-\phi'']_+$ change by at most
$5.319\times10^{-8}$ and $2.331\times10^{-8}$, respectively.
All measured curvatures are negative on the 129 primary questions, and
this holds for 254 of 256 questions overall.
The primary median sampled derivative extrema are
$1.789\times10^{-4}$ for $|\phi'''|$ and 0.001847 for
$[-\phi'']_+$.  No continuous derivative upper bounds, global
optimality certificate, or bound-selected scale are assigned to Huginn.

\paragraph{Scale selected from the sufficient range.}
For Ouro questions with $A>0$, we use the envelope value as $L_D$ in
Corollary~\ref{cor:beneficial-step} and fix
$\alpha_{\mathrm{safe}}=\min\{1,0.9(2A/L_D)\}$ before evaluating
the intervention.  The factor 0.9 is unchanged across questions.
The reported gains and the quarter-step comparator both use the same
FP64 readout.  All 315 primary gains exceed the computed sufficient
lower bound $\alpha_{\mathrm{safe}}A-L_D\alpha_{\mathrm{safe}}^2/2$,
with minimum numerical margin $4.680\times10^{-8}$.
Table~\ref{tab:curvature-selected-scale} reports the primary and
finite-step-failure populations separately, with 2000 paired
question-bootstrap resamples and seed 20260904.
On the 179 failures, the median selected scale is 0.32759, with
interquartile range $[0.15404,0.50776]$ and range
$[0.005902,0.821919]$.  Of the 179 positive gains, 142 exceed
$10^{-4}$ and 37 do not.
This construction uses reference-dependent derivatives along the
complete path and is evaluated before further recurrence.

\begin{table}[t]
\centering
\caption{\textbf{Selecting a scale from the negative-curvature bound.} Ouro-1.4B/MATH-500 with $\alpha_{\mathrm{safe}}=\min\{1,0.9(2A/L_D)\}$. Both actions are evaluated in FP64 on the same records, and counts require positive utility gain. Gain columns are in units of $10^{-3}$ with paired 95\% question-bootstrap intervals.}
\label{tab:curvature-selected-scale}
\small
\setlength{\tabcolsep}{5pt}
\begin{tabular}{@{}lrrrrr@{}}
\toprule
Population & $n$ & \makecell{Quarter-step\\positive gains} & \makecell{Bound-selected\\positive gains} & \makecell{Mean gain\\{}[95\% CI]} & \makecell{Gain over quarter\\{}[95\% CI]} \\
\midrule
Primary & 315 & 268 & \textbf{315} & \makecell[r]{$5.1985$\\$[4.2777,6.1934]$} & \makecell[r]{$\boldsymbol{2.6278}$\\$[1.9928,3.3228]$} \\
Finite-step failures & 179 & 132 & \textbf{179} & \makecell[r]{$1.0738$\\$[0.8555,1.3354]$} & \makecell[r]{$\boldsymbol{0.1442}$\\$[0.0610,0.2317]$} \\
\bottomrule
\end{tabular}
\vspace{-4mm}
\end{table}

\subsection{Immediate Quadratic Scale Selection on the Complete GSM8K Test Set}
\label{app:quadratic-intervention}

This additional analysis evaluates the quadratic rule of
\S~\ref{sec:optimal-step} on all 1319 GSM8K test questions with
Ouro-1.4B at $4\to5$.  Both actions share the original $H_4,H_5$
pair.  The quarter step uses scale 0.25; the quadratic rule maximises
$A\alpha+Q\alpha^2$ on $[0,1]$, with ties choosing the smaller scale.
The selected scale depends only on the original $A,Q$ and is fixed
before the intervention utility is evaluated.

Interpolation uses FP32 over all valid token positions, with padding
retaining the original proposed state.  The resulting state is cast to
BF16 before evaluation by the FP32 affine readout, with ordered FP64
utility reductions.  The utility uses the reference-token prediction
positions defined in Appendix~\ref{app:utilities}.
Table~\ref{tab:quadratic-intervention} reports the measured iteration-5
gain relative to the common $H_4$, rather than the ideal interpolant's
gain.  Mean utility gains use all 1319 questions; recovery counts use
the 431 finite-step failures of the original update.  Intervals use
2000 paired question-bootstrap resamples.

\begin{table}[t]
\centering
\caption{\textbf{Immediate utility under quadratic scale selection.} Ouro-1.4B/GSM8K test, $4\to5$. Mean gains relative to $H_4$ use all 1319 questions and are in units of $10^{-3}$, with paired 95\% question-bootstrap intervals. Recovery counts use the 431 original finite-step failures. The last row reports the paired utility advantage and net recovery differences. Outcomes use the computed BF16 intervention states.}
\label{tab:quadratic-intervention}
\small
\setlength{\tabcolsep}{8pt}
\begin{tabular}{@{}lrrr@{}}
\toprule
Action & \makecell{Mean gain\\{}[95\% CI]} & Recovered / 431 & Recovery (\%) \\
\midrule
Quarter step & \makecell[r]{$-0.839$\\$[-1.087,-0.580]$} & 306 / 431 & 71.00 \\
Quadratic rule & \makecell[r]{$\boldsymbol{+2.770}$\\$[+2.438,+3.146]$} & \textbf{406} / 431 & \textbf{94.20} \\
\midrule
Quadratic $-$ quarter & \makecell[r]{$\boldsymbol{+3.610}$\\$[+3.366,+3.881]$} & $\boldsymbol{+100}$ & $+23.20$ pp \\
\bottomrule
\end{tabular}
\vspace{-4mm}
\end{table}

\subsection{Question-Level Interventions Across HellaSwag Transitions}
\label{app:hellaswag-transitions}

This evaluation applies the quadratic rule to Ouro-2.6B on all 10042
HellaSwag validation questions at $1\to2$, $2\to3$, $3\to4$, and
$4\to5$.  Each question uses the joint option utility in
Eq.~(\ref{eq:option-utility}), with one common scale across all four
candidates.  The derivatives include their coupling through the
option-score log-softmax.  Each transition uses its own pair of
unmodified recurrent states and is evaluated immediately after scaling,
before further recurrence.  The first three transitions lie within the
four-step inference configuration, and the last extends it by one step.

The recurrent endpoints use BF16.  Candidate scores use the FP32
readout, followed by FP64 normalisation across options.  The quadratic
scale is computed from the original $A,Q$, maximising $q$ on $[0,1]$
with ties choosing the smaller scale.  All questions enter every
transition, giving 40168 question--transition records.  The joint scalar
second derivative and Hessian--vector calculation agree in all 64
numerical checks under the stated tolerance.  The original depth-selection
readout is reproduced bitwise in all 5021 batches, and the four-step
prefix is unchanged when computing the fifth state.

\begin{table}[t]
\centering
\caption{\textbf{Quadratic scale selection across HellaSwag transitions.} Ouro-2.6B, all 10042 validation questions, with one scale shared by the four candidates. (a) Recovery counts use the original finite-step failures. Mean gains use all questions, in units of $10^{-3}$. (b) Full-step sign accuracy and paired quadratic-minus-full gain with 95\% question-bootstrap intervals. Each transition is an independent intervention before further recurrence.}
\label{tab:hellaswag-transitions}
\small
\setlength{\tabcolsep}{4pt}
\begin{tabular*}{\linewidth}{@{\extracolsep{\fill}}lrrrrr@{}}
\toprule
\multicolumn{6}{@{}l}{\textit{(a) Recovery and mean immediate gain}} \\
Transition & Failures & \makecell{Recovered\\quarter $\to$ quadratic} & Full & Quarter & Quadratic \\
\midrule
$1\to2$ & 322 & $231\to\boldsymbol{297}$ & 59.317 & 17.370 & \textbf{99.763} \\
$2\to3$ & 209 & $162\to\boldsymbol{184}$ & 2.337 & -5.243 & \textbf{37.436} \\
$3\to4$ & 248 & $197\to\boldsymbol{242}$ & 4.799 & 0.731 & \textbf{18.956} \\
$4\to5$ & 88 & $67\to\boldsymbol{88}$ & 1.492 & 0.167 & \textbf{9.296} \\
\bottomrule
\end{tabular*}
\par\vspace{3pt}
\begin{tabular*}{\linewidth}{@{\extracolsep{\fill}}lrrr@{}}
\multicolumn{4}{@{}l}{\textit{(b) Endpoint prediction and paired gain}} \\
Transition & Sign $A$ (\%) & Sign $A+Q$ (\%) & Quadratic $-$ full [95\% CI] \\
\midrule
$1\to2$ & \textbf{85.17} & 82.44 & $\boldsymbol{40.446}\;[38.531,42.579]$ \\
$2\to3$ & 84.43 & \textbf{91.85} & $\boldsymbol{35.099}\;[33.948,36.268]$ \\
$3\to4$ & 87.04 & \textbf{95.94} & $\boldsymbol{14.157}\;[13.655,14.702]$ \\
$4\to5$ & 94.24 & \textbf{99.49} & $\boldsymbol{7.804}\;[7.508,8.113]$ \\
\bottomrule
\end{tabular*}
\end{table}

Quadratic selection recovers more finite-step failures than the quarter
step at each transition.  It also gives higher observed mean utility
gain over all questions than either the full or quarter step
(Table~\ref{tab:hellaswag-transitions}). At $4\to5$, the rule recovers
all 88 recorded finite-step failures, including 62 with gain
above $10^{-4}$.  The first three transitions likewise contain recoveries
above this margin, with 278, 170, and 211 cases, respectively.
These observations extend step-contraction recovery to a joint
multiple-choice utility at several recurrent depths.

Endpoint prediction and intervention measure different outcomes.
At $1\to2$, $A+Q$ has lower sign accuracy than $A$, although quadratic
selection still recovers 297 of 322 finite-step failures.  At $2\to3$,
quadratic selection recovers more failures, while the mean gain over
the original failure subset is 0.009511, compared with 0.009937 for the
quarter step.  Recovery counts and mean gains therefore retain their
separate populations and interpretations.  Reported paired intervals
use 2000 question-bootstrap resamples with seed 20260904.

\subsection{Token-Level Audit of Early Recurrent Updates}
\label{app:native-continue}

This audit evaluates Ouro-1.4B on all 1319 GSM8K test questions at
$1\to2$, $2\to3$, and $3\to4$, using the evaluated checkpoint's
original four-step depth-selection configuration.  The CDF threshold is
1.0, with no fixed-exit override.  All 167606 reference prediction
positions continue through the three transitions and exit at depth 4.
The gate's training-stage provenance is unresolved, and the results are
attributed to this checkpoint and configuration.

The original policy selects readout depth from a dense teacher-forced
trajectory.  Reconstructed selected states match the original readout
inputs bitwise in all 660 batches.  Gate probabilities and the CDF use
the original BF16 arithmetic; an FP64 reconstruction gives the same exit
decisions.  This audit therefore measures the geometry of continuing
updates under the stored policy, without a population of early exits
against which to assess gate discrimination.

At each reference prediction position, utility is the log probability
of the reference next token.  The coefficients $A,Q$ and gain
$\Delta U$ are computed for this scalar utility along that position's
fixed displacement, using the FP32 affine readout.  The quarter step
uses scale 0.25, and the quadratic rule maximises its local model on
$[0,1]$.  Each transition is evaluated independently from the original
trajectory; the scaled states are read out immediately, with no further
recurrence or cumulative intervention.  Prompt positions contribute to
policy coverage counts but not to the population used to evaluate reference utility.

Mean gains average over continuing prediction tokens.  Failure rates
use the same population, and recovery rates condition on the original
finite-step failures in each transition.  Confidence intervals use
2000 question-cluster bootstrap resamples with seed 20260904: all tokens
of a resampled question are retained together, and each replicate
recomputes the ratio of summed quantities to summed token counts.
These token-weighted estimands differ from the question-weighted
comparisons in the original evaluation sets.

\begin{table}[t]
\centering
\caption{\textbf{Finite-step failures in early recurrent updates.} Ouro-1.4B/GSM8K test, 1319 questions and 167606 continuing reference prediction positions per transition. (a) Failure rates use all continuing positions; recovery rates use the original failures. (b) Mean immediate gains and paired advantages are token-weighted reference log-probability increments, in units of $10^{-3}$. Paired 95\% intervals resample questions, retaining their tokens together. Original CDF threshold: 1.0.}
\label{tab:native-continue}
\small
\setlength{\tabcolsep}{3pt}
\begin{tabular*}{\linewidth}{@{\extracolsep{\fill}}lrrr@{}}
\toprule
\multicolumn{4}{@{}l}{\textit{(a) Failure counts and recovery}} \\
Transition & Finite-step failures (\%) & Quarter recovered (\%) & Quadratic recovered (\%) \\
\midrule
$1\to 2$ & 15657 (9.34) & 12479 (79.70) & \textbf{15517} (99.11) \\
$2\to 3$ & 16621 (9.92) & 12924 (77.76) & \textbf{16377} (98.53) \\
$3\to 4$ & 14473 (8.64) & 10879 (75.17) & \textbf{14150} (97.77) \\
\bottomrule
\end{tabular*}
\par\vspace{3pt}
\begin{tabular*}{\linewidth}{@{\extracolsep{\fill}}lrrrrr@{}}
\multicolumn{6}{@{}l}{\textit{(b) Utility over all continuing prediction positions}} \\
Transition & Full & Quarter & Quadratic & \makecell{Quadratic $-$ full\\{}[95\% CI]} & \makecell{Quadratic $-$ quarter\\{}[95\% CI]} \\
\midrule
$1\to 2$ & 637.812 & 305.453 & \textbf{651.649} & $\boldsymbol{13.837}\;[12.467,15.184]$ & $\boldsymbol{346.195}\;[339.764,352.695]$ \\
$2\to 3$ & 142.276 & 66.697 & \textbf{183.909} & $\boldsymbol{41.633}\;[40.670,42.605]$ & $\boldsymbol{117.212}\;[114.934,119.489]$ \\
$3\to 4$ & 18.271 & 12.028 & \textbf{57.044} & $\boldsymbol{38.773}\;[37.941,39.571]$ & $\boldsymbol{45.016}\;[44.100,45.946]$ \\
\bottomrule
\end{tabular*}
\vspace{-4mm}
\end{table}

All three transitions have positive mean full-step gain, while
8.64--9.92\% of their continuing prediction positions exhibit
finite-step failure.  The quadratic rule recovers 97.77--99.11\% of
these failures and has positive paired mean advantages over both full
and quarter steps on the entire continuing population
(Table~\ref{tab:native-continue}).  A descriptive margin check requiring
$A>10^{-4}$ and $\Delta U<-10^{-4}$ retains 13938, 13002, and 9706
failures, respectively, within the original populations.  These results
locate recoverable utility losses in early recurrent updates with
positive aggregate progress, providing a further test of the
direction--step distinction within the four-step inference budget.

\paragraph{Future direction: geometry-guided adaptive computation.}
\label{app:geometry-adaptive}
These findings motivate using reference-defined geometry as offline
supervision for learned exit and step-scale decisions, complementing
task-loss-based depth allocation in Ouro~\citep{zhu_scaling_2026}.
Learning such decisions from inference-time states and evaluating them
on generated trajectories under matched computation budgets would
connect this mechanism to task-aware adaptive computation.

\subsection{Curvature-Bound Evaluation on the Complete Mathematical Benchmarks}
\label{app:cloud-bounds}

The bound evaluation covers both Ouro models on the complete MATH-500
test and GSM8K test sets, with 3638 model--question records and 2034
primary records satisfying $A>0,Q<0$.  All four conditions use
$4\to5$, freshly computed CUDA BF16 endpoints at batch size two,
and FP64 readout analysis with 257 partition nodes.
These are separate computed paths from the original evaluation
sets.  The Ouro-1.4B/MATH-500 primary count of 315 equals that of
the original evaluation set.  Predictions, derivatives, and numerical
optimum estimates are computed within this run.
Table~\ref{tab:cloud-bound-comparison} compares three scale bounds
for the same predictions: the global-third-derivative bound
$M/(2\kappa)$, the integrated bound $\overline C/\kappa$, and the
signed-residual bound $\overline S/\kappa$ defined below.

\begin{table}[t]
\centering
\caption{\textbf{Scale bounds on the complete mathematical benchmarks.} Median and, in brackets, 75th percentile of each bound on $A>0,Q<0$, before the minimum with one in Theorem~\ref{thm:integrated-step}. Coverage counts bounds covering the original numerical scale-error upper endpoints; refinement of five cases is reported in Table~\ref{tab:bound-reference-refinement}. All four conditions use separately computed CUDA endpoints, $4\to5$, and FP64 analysis.}
\label{tab:cloud-bound-comparison}
\small
\setlength{\tabcolsep}{4pt}
\begin{tabular}{@{}lrrrrrr@{}}
\toprule
Model / task & $n$ & $M/(2\kappa)$ & $\overline C/\kappa$ & $\overline S/\kappa$ & $\overline C$ covers & $\overline S$ covers \\
\midrule
\makecell[l]{Ouro-1.4B\\MATH-500} & 315 & \makecell[r]{4.6108\\$[8.1823]$} & \makecell[r]{0.0967\\$[0.1744]$} & \makecell[r]{0.0852\\$[0.1578]$} & 315/315 & 315/315 \\
\makecell[l]{Ouro-1.4B\\GSM8K} & 683 & \makecell[r]{3.6661\\$[4.6757]$} & \makecell[r]{0.0583\\$[0.0974]$} & \makecell[r]{0.0485\\$[0.0900]$} & 683/683 & 680/683 \\
\makecell[l]{Ouro-2.6B\\MATH-500} & 266 & \makecell[r]{2.8436\\$[4.5339]$} & \makecell[r]{0.0602\\$[0.1031]$} & \makecell[r]{0.0520\\$[0.0949]$} & 266/266 & 265/266 \\
\makecell[l]{Ouro-2.6B\\GSM8K} & 770 & \makecell[r]{2.8903\\$[4.0521]$} & \makecell[r]{0.1136\\$[0.1734]$} & \makecell[r]{0.1042\\$[0.1677]$} & 770/770 & 769/770 \\
\bottomrule
\end{tabular}
\vspace{-4mm}

\end{table}

The integrated median lies between 0.0583 and 0.1136 and the signed
median between 0.0485 and 0.1042, compared with 2.84--4.61 for the
global bound.  The integrated and signed bounds also have lower 75th
percentiles in every condition.
The integrated bound covers the numerical scale-error interval of all
2034 primary records.  The signed bound covers 2029 records directly
and the remaining five after the numerical refinement described below.
These comparisons give tighter bounds for the same predictions and do
not change the prediction errors.
The four cell summaries match the aggregate summary exactly.

\paragraph{Distribution of the integrated bounds.}
With cellwise third-derivative bounds $L_\ell$ on the partition
$0=t_0<\cdots<t_N=1$, the upper sum of
Proposition~\ref{prop:curvature-partition} is evaluated as
\begin{equation}
  \label{eq:integrated-bound-evaluated}
  \overline C
  =\sum_{\ell=0}^{N-1}\left[
    \frac{h_\ell}{2}\bigl(|g(t_\ell)|+|g(t_{\ell+1})|\bigr)
    +\frac{L_\ell h_\ell^2}{4}\right],
  \qquad g(s)=\phi''(s)-\phi''(0),
\end{equation}
and enters the scale and regret bounds of Theorem~\ref{thm:integrated-step}.
Table~\ref{tab:integrated-bound-distribution} reports the distribution
of the scale bound $\overline C/\kappa$ and the regret bound
$\overline C^{2}/(2\kappa)$ over the 315 primary Ouro-1.4B/MATH-500
records, together with the numerical regret upper endpoints.  The scale
bound has median 0.097 and 75th percentile 0.174.  The regret
bound has median $1.04\times10^{-4}$, between two and three orders of
magnitude above the median numerical regret and one order of magnitude
below the magnitude of the mean utility change of the primary subset,
$-1.13\times10^{-3}$.  The source summary records these distributions
only, so the comparison of the regret bound with the numerical regret
on individual records is not reported.

\begin{table}[t]
\centering
\caption{\textbf{Distribution of the integrated bounds on Ouro-1.4B/MATH-500.} The 315 primary records ($A>0,Q<0$). Scale bound $\overline C/\kappa$, regret bound $\overline C^{2}/(2\kappa)$, and upper endpoint of the numerical regret $\phi(\alpha^\star)-\phi(\widehat\alpha)$, before the minimum with one in Theorem~\ref{thm:integrated-step}. The regret columns compare distributions, not individual records.}
\label{tab:integrated-bound-distribution}
\small
\setlength{\tabcolsep}{4pt}
\begin{tabular}{@{}lrrr@{}}
\toprule
Estimate & $\overline C/\kappa$ & $\overline C^{2}/(2\kappa)$ & Numerical regret upper \\
\midrule
Mean & 0.1448 & $8.747\times10^{-4}$ & $1.831\times10^{-5}$ \\
25th percentile & 0.0556 & $3.163\times10^{-5}$ & $1.894\times10^{-9}$ \\
Median & 0.0967 & $1.037\times10^{-4}$ & $2.192\times10^{-7}$ \\
75th percentile & 0.1744 & $3.090\times10^{-4}$ & $6.113\times10^{-6}$ \\
Maximum & 2.7572 & $5.877\times10^{-2}$ & $1.452\times10^{-3}$ \\
\bottomrule
\end{tabular}
\vspace{-4mm}
\end{table}

\paragraph{Signed-residual envelope.}
Let $e(a):=\phi'(a)-q'(a)$ on $[0,1]$.  Under the $C^3$ setup,
$e''=\phi'''$.  Use the partition and valid nonnegative constants
$L_\ell$ of Proposition~\ref{prop:curvature-partition}, now obtained
from $|\phi'''|\leq L_\ell$ throughout each cell.  Define
\begin{equation}
  \label{eq:signed-residual-envelope}
  \overline S:=\min\left\{\overline C,
    \max_{0\leq\ell<N}\left[
      \max\{|e(t_\ell)|,|e(t_{\ell+1})|\}
      +\frac{L_\ell h_\ell^2}{8}\right]\right\}.
\end{equation}
\begin{proposition}[Signed-Residual Approximation Bounds]
  \label{prop:signed-residual-bounds}
  Under these hypotheses and $Q<0$, $|e(a)|\leq\overline S$ for
  every $a\in[0,1]$.  For every true global maximiser $\alpha^\star$,
  \begin{equation}
    |\widehat\alpha-\alpha^\star|\leq\min\{1,\overline S/\kappa\},
    \qquad
    0\leq\phi(\alpha^\star)-\phi(\widehat\alpha)
       \leq\frac{\overline S^2}{2\kappa}.
  \end{equation}
\end{proposition}
\begin{proof}
  On a cell $[l,r]$ with $|e''|\leq L$, let $I_e$ be the linear
  interpolant of the endpoint values of $e$.  The functions
  $e(a)+La^2/2$ and $-e(a)+La^2/2$ are convex.  Bounding each by
  its endpoint interpolant and subtracting the quadratic gives
  \begin{equation}
    |e(a)-I_e(a)|\leq\frac{L}{2}(a-l)(r-a)
       \leq\frac{L(r-l)^2}{8}.
  \end{equation}
  Since $|I_e(a)|\leq\max\{|e(l)|,|e(r)|\}$, the cell expression
  bounds $|e|$.  Every point is covered by a partition cell, so taking
  the maximum gives a uniform residual bound.  Independently,
  $|e(a)|\leq C(1)\leq\overline C$, so their minimum remains valid.
  The variational location bound and the residual-increment/quadratic-gap
  argument in Appendix~\ref{app:integrated-bounds-proof} apply with
  $\varepsilon=\overline S$.
\end{proof}

\paragraph{Refinement of the numerical optimum estimate.}
The original summary records integrated-bound coverage of all 2034
primary error-interval upper endpoints and signed-bound coverage of
2029.  In the remaining five cases, the signed bound lies inside the
original numerical error interval.  We refine those optimum intervals
using the stored endpoint derivatives and curvature bounds, keeping the
predicted scale and its bound fixed.
For an original root bracket $[l,u]$, suppose
$\phi'(l)>0\geq\phi'(u)$ and
$-\nu\leq\phi''(a)\leq-\mu<0$ throughout the bracket.
The unique derivative zero belongs to the intersection of $[l,u]$ and
\begin{equation}
  \left[l+\frac{\phi'(l)}{\nu},\ l+\frac{\phi'(l)}{\mu}\right],
  \qquad
  \left[u+\frac{\phi'(u)}{\mu},\ u+\frac{\phi'(u)}{\nu}\right].
\end{equation}
These intervals follow by integrating the negative-curvature bounds
from each endpoint.  Their construction is independent of the
predicted scale and its error bound.  All five refined error upper
endpoints lie below the unchanged signed bounds
(Table~\ref{tab:bound-reference-refinement}).
The original aggregate counts are retained, with this refinement
reported as a supplementary check.  It requires no additional model
evaluation.  The full-matrix metrics use the cloud summaries.  Local
record-level recomputation covers these five exported cases.

\begin{table}[t]
\centering
\caption{\textbf{Refinement of five numerical optimum intervals.} The prediction and signed scale bound are unchanged. Endpoint derivatives and local curvature bounds refine the numerical optimum interval. Margin is bound minus refined error upper, in units of $10^{-5}$. Rows are identified by source row index.}
\label{tab:bound-reference-refinement}
\small
\setlength{\tabcolsep}{4pt}
\begin{tabular}{@{}lrrrr@{}}
\toprule
Case / source row & \makecell{Original\\error upper} & \makecell{Unchanged\\$\overline S/\kappa$} & \makecell{Refined\\error upper} & Margin \\
\midrule
Ouro-1.4B / GSM8K / 00670 & 0.00926876 & 0.00925708 & 0.00921537 & 4.171 \\
Ouro-1.4B / GSM8K / 00716 & 0.03283225 & 0.03282935 & 0.03280362 & 2.573 \\
Ouro-1.4B / GSM8K / 01136 & 0.00721219 & 0.00718848 & 0.00716065 & 2.783 \\
Ouro-2.6B / GSM8K / 00909 & 0.03655533 & 0.03654823 & 0.03650648 & 4.175 \\
Ouro-2.6B / MATH-500 / 00005 & 0.01955370 & 0.01955096 & 0.01949695 & 5.401 \\
\bottomrule
\end{tabular}
\vspace{-4mm}
\end{table}

\subsection{Oracle Gains}
\label{app:exp-oracles}

The oracle comparison evaluates every question in its specified
evaluation set, using the five-point action set $\mathcal B$ for both the
gain over the full step and the gain over endpoint selection.
Strict interior advantage requires an available interior scale to
exceed both endpoint utilities.  The final column of
Table~\ref{tab:oracle-comparison} is the ratio of mean gain over
endpoint selection to mean gain over the full step.
The mean advantage of intermediate scales accounts for 5.65\% of the
five-scale oracle's gain over the full step on Ouro/MATH-500, and less
than 0.5\% on Huginn/GSM8K and Ouro/HellaSwag.
For Huginn/CommonsenseQA, six questions have strict interior advantage,
but none exceed $10^{-4}$, and its small mean advantage is retained at
sufficient precision in the table.

\begin{table}[t]
\centering
\caption{\textbf{Intermediate scales beyond endpoint selection.} Gain is $U_{\mathrm{step}}-U_{\mathrm{halt}}$ on the five-point action set, in units of $10^{-4}$ with 95\% intervals. Interior counts require strict improvement over both endpoints. Share is mean gain over endpoint selection divided by mean gain over the full step.}
\label{tab:oracle-comparison}
\small
\setlength{\tabcolsep}{4pt}
\begin{tabular}{@{}llrrrr@{}}
\toprule
Model & Evaluation set & $N$ & Interior & Gain over halt [95\% CI] & Share (\%) \\
\midrule
Ouro-1.4B & MATH-500 test & 500 & 211 & $6.902\;[5.762,8.169]$ & 5.65 \\
Huginn & GSM8K train256 & 256 & 26 & $0.120\;[0.065,0.182]$ & 0.44 \\
Ouro-1.4B & HellaSwag validation1000 & 1000 & 44 & $0.948\;[0.560,1.369]$ & 0.46 \\
Ouro-1.4B & CommonsenseQA train1024 & 1024 & 32 & $1.260\;[0.501,2.200]$ & 0.08 \\
Ouro-2.6B & CommonsenseQA train1024 & 1024 & 31 & $2.723\;[1.418,4.323]$ & 0.15 \\
Huginn & CommonsenseQA train1024 & 1024 & 6 & $0.00153\;[0.00009,0.00348]$ & 0.00195 \\
\bottomrule
\end{tabular}
\vspace{-4mm}
\end{table}

\section{Limitations}
\label{app:limitations}

The study examines individual recurrent updates in frozen looped
Transformers, using reference utility under teacher forcing to measure
task progress.  Each intervention scales an already computed proposal
and evaluates the resulting state before further recurrence.  This
isolates the immediate effect of update magnitude along the model's
direction.  Applying such interventions during generation would require
evaluating their effects on subsequent states, answer quality, and
computation cost.  The present measurements identify recoverable losses
within individual updates; how these losses accumulate over a recurrent
trajectory remains to be established.

\end{document}